\documentclass[twoside,11pt]{article}

\usepackage{jmlr2e}
\usepackage{hyperref}       % hyperlinks
\usepackage{url}            % simple URL typesetting
\usepackage{booktabs}       % professional-quality tables
\usepackage{amsfonts}       % blackboard math symbols
\usepackage{nicefrac}       % compact symbols for 1/2, etc.
\usepackage{amsmath}
\usepackage{graphicx}
\usepackage{xcolor}
\usepackage{epstopdf}
\usepackage{wrapfig}
\usepackage{subcaption}
\usepackage[]{algorithm2e}
\usepackage{epstopdf}
\usepackage{wrapfig}
\usepackage{lastpage}

\usepackage{ulem}
\usepackage{algorithmic}
\graphicspath{ {images/} }

\graphicspath{ {images/} }
\newcommand{\R}{\mathbb{R}}

\let\bbl\Bigl

\let\bbbbl\Biggl

\let\bbr\Bigr

\let\bbbbr\Biggr
\newcommand{\norm}[1]{\lVert{#1}\rVert}

\newcommand{\normmm}[1]{\bbl\lVert{#1}\bbr\rVert}

\newcommand{\normmmmm}[1]{\bbbbl\lVert{#1}\bbbbr\rVert}

\newcommand{\bmat}[1]{\begin{bmatrix}#1\end{bmatrix}}

\newtheorem{thm}{Theorem}
\newtheorem{prop}[thm]{Proposition}
\newtheorem{lem}[thm]{Lemma}
\newtheorem{defn}[thm]{Definition}
\newtheorem{corr}[thm]{Corollary}
\newcommand{\ip}[2]{\langle #1, #2 \rangle}
\newcommand{\mcl}[1]{\mathcal{#1}}
\newcommand{\mbf}[1]{\mathbf{#1}}
\newcommand{\half}{\frac{1}{2}}
\renewcommand{\S}{\mathbb{S}}
\newcommand{\N}{\mathbb{N}}

\jmlrheading{25}{2024}{1-\pageref{LastPage}}{4/23; Revised 1/24}{6/24}{23-0528}{Aleksandr Talitckii, Brendon Colbert and Matthew M. Peet}

\ShortHeadings{Algorithm for Universal Kernel Learning}{Talitckii, Colbert and Peet}
\firstpageno{1}

\begin{document}

\title{Efficient Convex Algorithms for Universal Kernel Learning}

\author{
	\name Aleksandr Talitckii \email atalitck@asu.edu \\
       \addr Department of Mechanical and Aerospace Engineering\\
       Arizona State University\\
       Tempe, AZ 85281-1776, USA
       \AND	
	\name Brendon Colbert \email bkcolbe1@asu.edu \\
       \addr Department of Mechanical and Aerospace Engineering\\
       Arizona State University\\
       Tempe, AZ 85281-1776, USA
       \AND
       \name Matthew M. Peet \email mpeet@asu.edu \\
       \addr Department of Mechanical and Aerospace Engineering\\
       Arizona State University\\
       Tempe, AZ 85281-1776, USA}

\editor{Ambuj Tewari}

\maketitle

\begin{abstract}%   <- trailing '%' for backward compatibility of .sty file
The accuracy and complexity of machine learning algorithms based on kernel optimization are determined by the set of kernels over which they are able to optimize. An ideal set of kernels should: admit a linear parameterization (for tractability); be dense in the set of all kernels (for robustness); be universal (for accuracy). Recently, a framework was proposed for using positive matrices to parameterize a class of positive semi-separable kernels. Although this class can be shown to meet all three criteria, previous algorithms for optimization of such kernels were limited to classification and furthermore relied on computationally complex Semidefinite Programming (SDP) algorithms. In this paper, we pose the problem of learning semiseparable kernels as a minimax optimization problem and propose a SVD-QCQP primal-dual algorithm which dramatically reduces the computational complexity as compared with previous SDP-based approaches. Furthermore, we provide an efficient implementation of this algorithm for both classification and regression -- an implementation which enables us to solve problems with 100 features and up to 30,000 datums. Finally, when applied to benchmark data, the algorithm demonstrates the potential for significant improvement in accuracy over typical (but non-convex) approaches such as Neural Nets and Random Forest with similar or better computation time.
\end{abstract}

\begin{keywords}
  kernel functions, multiple kernel learning, semi-definite programming, supervised learning, universal kernels
\end{keywords}

\section{Introduction}
{ 

%Kernels methods have shown robust and accurate results for classification and regression problems for the last 30 years. However, such 
%Kernel methods (and Support Vector Machines (SVMs) in particular) require selection of a kernel. 

%The typical approach is to consider the 

%In practice, different machine learning problems also demonstrate the importance of selection of a kernel. For example, the breast-cancer detection problem can be efficiently solved with SVM with sigmoid kernel~\citep[see][]{}, while 
%While sometimes it is possible to propose a reasonable kernel from an intuitive understanding of the data set, typically it is hard to propose or even to guess a reasonable kernel even if the structure of the data set is given. 

%All of the above implies that the selection of a kernel function is one of the most important questions for kernels methods. As it is shown in~\cite{hussain2011comparison} a sigmoid kernel function demonstrates the best specificity for the breast-cancer detection. In both of these results authors compared polynomial, RBF, sigmoid and linear kernels to choose the best kernel function according to some metrics. 

Kernels allow for a convex formulation of the nonlinear classification and regression problems -- improving accuracy and robustness of data-based modeling. Specifically, every kernel defines a feature map of the data to a Reproducing Kernel Hilbert Space (RKHS) wherein a linear classification or regression problem may be solved. Furthermore, if the kernel is universal, the RKHS will be infinite dimensional -- implying that, e.g. classification data will always be linearly separable in the infinite-dimensional RKHS.  % Common classes of universal kernels include the Gaussian, sigmoid, and Laplacian. 

While any universal kernel (for example Gaussian or Laplacian kernel) will yield a linear separation in its associated RKHS, the robustness of that separation will be strongly influenced by the topology of the RKHS. For a poorly chosen kernel, the resulting fit will be sensitive to noise and hence the classifier or regressor may perform poorly on untrained data. However, for a well chosen kernel, the separation of data will be robust -- yielding improved performance on untrained data. For example, when considering kernel selection in cancer diagnosis~\citep[See][]{misc_breast_cancer_wisconsin} (a problem with significant variations in data collection mechanisms), lack of robustness of the classifier may result in incorrectly labelling a malignant tumour as benign. While rigorous numerical experimentation has been used to find suitable kernels for well-studied problems such as cancer classification~\citep[See][]{hussain2011comparison}, when the underlying data generating mechanism is new or speculative, kernel selection is itself an optimization problem known as learning the kernel. 
%However, formulating and solving the kernel learning problem efficiently is challenging and requires parameterization of a sufficiently rich set of kernels.
Specifically, Kernel Learning (KL) algorithms~\citep[such as those found in][]{xu2010simple,sonnenburg2010shogun,yang2011efficient} have been proposed to find the kernel, $k \in \mcl K$ which optimizes an achievable metric such as the soft margin (for classification). However, the set of kernels, $k \in \mcl K$, over which the algorithm can optimize strongly influences the performance and robustness of the resulting classifier or predictor. }

To understand how the choice of kernel influences performance and robustness, several properties of positive kernels have been considered. For example, a \textit{characteristic} kernel was defined in~\cite{fukumizu2007kernel} to be a kernel whose associated integral operator is injective. Alternatively, a kernel, $k$, is \textit{strictly positive definite} if the associated integral operator has trivial null-space  -- See, e.g.~\cite{steinwart2008support}. As stated in, e.g.~\cite{steinwart2001influence}, it has been observed that kernels with the characteristic and strictly positive definite properties are able to perform arbitrarily well on large sets of training data. Similar to characteristic and strictly positive definite kernels, perhaps the most well-known kernel property is that of $c_0$-universality, which implies that the associated RKHS is dense in $\mcl C$.
The relationship between these and other kernel properties (using, e.g. alternative norms) has been studied extensively in, e.g.~\cite{sriperumbudur2011universality}. In particular, the universal property implies the kernel is both characteristic and strictly positive definite, while under certain conditions the converse also holds and all three properties are equivalent~\citep{simon2018kernel}. As a result of these studies, there is a consensus that in order to be effective on large data sets, a kernel should have the universal property.

While the universal property is now well-established as being a desirable property in any kernel, much less attention has been paid to the question of what are the desirable properties of a set of kernels. This question arises when we use kernel learning algorithms to find the optimal kernel in some parameterized set. The assumption, of course, is that for any given data generating process, there is some ideal kernel whose associated feature map maximizes separability of data generated by that underlying process. Clearly, then, when constructing a set of kernels to be used in kernel learning, we would like to ensure that this set contains the ideal kernel or at least a kernel whose associated feature map closely approximates the feature map of this ideal kernel. With this in mind,~\citep{JMLR} proposed three desirable properties of a set of kernels $\mcl K$ - tractability, density, and universality. Specifically, $\mcl K$ is said to be tractable if $\mcl K$ is convex (or, preferably, a linear variety) - implying the kernel learning problem is solvable in polynomial time \citep[e.g.][]{rakotomamonjy_2008,jain_2012,lanckriet_2004,qiu2005multiple}. The set $\mcl K$ has the density property if, for any $\epsilon>0$ and any positive kernel, $k^*$ there exists a $k \in \mcl K$ where $\norm{k-k^*}\le \epsilon$. The density property implies that kernels from this set can approximate the feature map of the ideal kernel arbitrarily well (which then implies the resulting learned kernel will perform well on untrained data). Finally, the set $\mcl K$ is said to have the universal property if any $k \in \mcl K$ is universal.

Having defined desirably properties in a set of kernels, the question becomes how to parameterize such a set. While there are many ways of parameterizing sets of kernels \citep[See][for a survey]{gonen2011multiple}, not all such parameterizations result in a convex kernel learning algorithm. Furthermore, at present, there is no tractable  parameterization which is dense in the set of all possible kernels. To address this problem, in~\cite{JMLR}, a general framework was proposed for using positive matrices and bases to parameterize families of positive kernels~\citep[as opposed to positive kernel matrices as in][]{lanckriet_2004,qiu2005multiple,ni2006learning}. This framework allows one to define a basis of kernels for a class of integral operators and then to use SDP to find kernels which represent squares of such operators -- implying that the resulting kernels define positive operators. In particular, \cite{JMLR} proposed a set of basis functions which were then used to parameterize positive integral operators of the form (using one-dimension for simplicity)
\begin{equation}
(\mcl P \mbf x)(s):=k_{1a}(s) \int_0^s k_{1b}(\theta)\mbf x(\theta)d \theta+k_{2a}(s)\int_s^1 k_{2b}(\theta)\mbf x(\theta)d \theta,\label{eqn:semisep}
\end{equation}
which correspond to what are known (in one-dimension) as semiseparable kernels -- See, e.g.~\cite{gohberg_book}. In $n$-dimensions, a kernel constructed using this particular parameterization was referred to as a Tessellated Kernel (TK) -- indicative of their blockwise partition of the domain (a feature resulting from the semi-separable structure and reminiscent of the activation functions used in neural nets). It was further shown that the interior of this class of kernels was universal and the set of such kernels was dense in the set of all positive kernels. Using this positive matrix parameterization of the family of Tessellated Kernels, it was shown in~\cite{JMLR} that the associated SVM kernel learning algorithm could be posed as an SDP and that the solution to this SDP achieves superior Test Set Accuracy (TSA) when compared with a representative sample of existing classification algorithms (including non-convex kernel learning methods such as simple Neural Nets).
%However, due to the use of SDP solvers as discussed above, the computational complexity of the algorithm proposed in~\cite{JMLR} was significantly higher than most of the algorithms to which it was compared. %Furthermore, the TK algorithm proposed in~\cite{JMLR} only considered classification and was not extended to regression.
%More explicitly, we consider  a generalized representation of the kernel learning problem, which encompasses both classification and regression where \citep[using the representer theorem as in][]{scholkopf2001generalized}, that can be represented as min-max problem 
%\begin{equation} \label{eqn:OPT}
%\max_{P \succeq 0} \min_{\alpha \in \R^m} h(\alpha) - \Phi(\alpha, P),
%\end{equation}
%where $h(\alpha)$ is a convex function and $\Phi(\alpha, P)$ is a convex in $\alpha$ and concave in $P$

Unfortunately, however, although the TSA data reported in~\cite{JMLR} showed improvement over existing classification algorithms, this accuracy came at the cost of significant increases in computational complexity -- a factor attributable to the high complexity of primal-dual interior point algorithms for solving SDP.
Unlike Quadratic Programming (QP), which is used to solve the underlying SVM or Support Vector Regression (SVR) problem for a fixed kernel, the use of SDP, Quadratically Constrained Quadratic Programming (QCQP) and Second Order Cone Programming (SOCP) for kernel learning significantly limits the amount of training data which can be processed. Note that while this complexity issue has been partially addressed in the context of MKL \citep[which considered an efficient reduction to QP complexity in][]{jain_2012}, such a reduction to QP has not previously been proposed for SDP-based algorithms such as kernel matrix learning.
The {  main} goal of this paper, then, is to propose a new algorithm for optimizing over families of kernels parameterized by positive matrices, but without the use of SDP and its associated computational overhead.

{  Fundamentally, the algorithm proposed in this paper is based on a reformulation of the SDP defined in~\cite{JMLR} as a saddle-point optimization (See Section~\ref{sec:Ideal} and Section~\ref{sec:GKL}). This saddle-point formulation allows us to then decompose the optimization problem into primal and dual sub-problems, $OPT\_A$ and $OPT\_P$ (Section~\ref{sec:2step}). Based on this decomposition, we propose a Franke-Wolfe type algorithm for solving the kernel learning problem - an approach based on the work in~\cite{rakotomamonjy_2008} and \cite{jain_2012}. Critically, we then show that the SDP in the subproblem $OPT\_P$ admits an analytic solution using the Singular Value Decomposition - implying a worst-case computational complexity of $O(n_P^3)$ where $n_P^2/2+n_P$ is the number of parameters in the family of kernels, $\mcl K$. In addition, we show that $OPT\_A$ is a convex QP and may be similarly solved with a complexity of $O(m^3)$ (or even $O(m^{2.3}$) for LibSVM implementation),  where $m$ is the number of data points. As a result, the resulting computational complexity of the proposed algorithm is dramatically reduced compared with the complexity of the SDP-based algorithms proposed in~\cite{JMLR} (with complexity~\citep{borchers2007implementation}  $O(m^4)$ with respect to $m$ and $O(n_P^6)$ with respect to $n_P$. Summarizing, the proposed algorithm does not require the use of SDP, QCQP or SOCP and, when applied to several standard test cases, has observed complexity which scales as $O(m^{2.3})$ or less.}

In addition to the proposed algorithm, this paper also extends the convex kernel learning framework proposed in~\cite{JMLR} to the problem of kernel learning for regression. The kernel learning problem in regression has been studied for kernel matrix learning as in~\cite{lanckriet_2004} and for MKL in e.g.~\cite{rakotomamonjy_2008,jain_2012}. However, the regression problem has not previously been considered using the generalized framework for kernel learning presented in~\cite{JMLR}. In this paper, we provide such extension and demonstrate significant increases in performance, as measured by both computation time and Mean Square Error (MSE) when compared to other optimization-based kernel learning algorithms as well as when compared to more heuristic approaches such as Random Forest and deep learning.
 
%Finally, we note that the algorithms proposed in this paper are not limited to the Tessellated class of kernels, but apply to any class of kernels which define positive operators using the generalized kernel learning framework as proposed by~\cite{JMLR}.

\section{Notation}

We use $\N, \R$ to denote the natural and real numbers, respectively. We use $\mathbf{1}^n\in \R^n$ to denoted the vector of ones. For $x\in\R^n$ we use $\|x\|_p$ for the $L_p$-norm and use $x \ge 0$ to indicate the positive orthant -- i.e. $x_i \ge 0$ for all $i$.  For $x, y\in \R^n$, $x\odot y$ denotes elementwise multiplication. {  The space of $n\times n$-square real symmetric matrices is denoted $\S^{n}$ with $\S^n_{+}\subset \S^n$ being the cone of positive semi-definite matrices. We use $\S^n_{++} \subset \S^n_+$ to denote the space of positive definite matrices. For a matrix $P\in S^n$, we use $P\succeq 0$ ($P\succ 0$) if $P$ is a positive semi-definite matrix (positive definite matrix).} For $A, B\in \R^{n\times n}$, $\ip{A}{B}$ denotes the Frobenius matrix inner product. %We denote $I, 0$ to be the identity and null matrices when dimensions are clear from context.
For a compact set $X \subset \R^n$, we denote $\mcl C(X)$ to be the space of scalar continuous functions defined on $X$ and equipped with the uniform norm $\|f\|_{\mcl C} := \sup_{x\in X} |f(x)|$. {  For a differentiable function with a single argument, we use $\nabla f(x)$ to denote the gradient of function $f$ at point $x$. For functions with two explicit arguments where we do not need to specify a point of differentiation, the partial gradient of function $f(x, y)$ is denoted $\nabla_x f(x, y)$. In cases where we need to specify both the argument and the point of differentiation, we use $\nabla_x f(x, y)\big|_{x = x_0}$. Furthermore, if $x \in \R^{n \times n}$ and $f(x, y) \in \R$,  we use $\nabla_x f(x, y) \in \R^{n\times n}$ to denote the matrix where $(\nabla_x f(x, y))_{ij} = \frac{\partial f(x, y)}{\partial x_{ij}} $. } For a given positive definite kernel $k \in \mcl C(X\times X)$, ${\mcl H}_k$ denotes the associated Reproducing Kernel Hilbert Space (RKHS), where the subscript $k$ is dropped if clear from context.

\section{Kernel Sets and Kernel learning} \label{sec:Ideal}
Consider a generalized representation of the kernel learning problem, which encompasses both classification and regression where \citep[using the representor theorem as in][]{scholkopf2001generalized} the learned function is of the form $f_{\alpha,k}(z) = \sum_{i=1}^m \alpha_i k(x_i,z)$.
\begin{equation} \label{eqn:OPT}
\min_{k  \in \mcl{K}} \min_{\substack{\alpha \in \R^m \\ b \in \R}} \norm{f_{\alpha,k}}^2 + C \sum\nolimits_{i=1}^m l(f_{\alpha,k},b)_{y_i,x_i}.
\end{equation}
Here $\norm{f_{\alpha,k}} = \sum_{i=1}^m \sum_{j=1}^m \alpha_i \alpha_j k(x_i,x_j)$ is the norm in the Reproducing Kernel Hilbert Space (RKHS) and $l(f_{\alpha,k},b)_{y_i,x_i}$ is the loss function defined for SVM binary classification and SVM regression as $l_c(f_{\alpha,k},b)_{y_i,x_i}$ and $l_r(f_{\alpha,k},b)_{y_i,x_i}$, respectively, where
\begin{align*}
l_c(f_{\alpha,k},b)_{y_i,x_i}  &=  \max \{  0, 1 - y_i (f_{\alpha,k}(x_i) - b) \}, \\
l_r(f_{\alpha,k},b)_{y_i,x_i} &=  \max \{  0, |y_i - (f_{\alpha,k}(x_i) - b)|- \epsilon \}.
\end{align*}

The properties of the classifier/predictor, $f_{\alpha,k}$, resulting from Optimization Problem~\ref{eqn:OPT} will depend on the properties of the set $\mcl K$, which is presumed to be a subset of the convex cone of all positive kernels. To understand how the choice of $\mcl K$ influences the tractability of the optimization problem and the resulting fit, we consider three properties of the set, $\mcl K$. These properties can be precisely defined as follows.

\subsection{Tractability}
We say a set of kernel functions, $\mcl K$, is tractable if it can be represented using a countable basis.
\begin{defn} \label{Tractability}
The set of kernels $\mcl K$ is \textit{\textbf{tractable}} if there exist a countable set $\{G_{i}(x,y)\}_{i=1}^{\infty}$ such that, for any $k \in \mcl K$, there exists $n_k \in \N$ where $k(x,y)=\sum_{i=1}^{n_k} v_i G_i(x,y)$ for some $v \in \R^{n_k}$.
\end{defn}
Note the $G_i(x,y)$ need not be positive kernel functions. The tractable property is required for the associated kernel learning problem to be solvable in polynomial time.

\subsection{Universality}\label{Universal_Motivation}
Universal kernel functions always have positive definite (full rank) kernel matrices, implying that for arbitrary data $\{y_i,x_i\}_{i=1}^m$, there exists a function $f(z) = \sum_{i=1}^m \alpha_i k(x_i,z)$, such that $f(x_j) = y_j$ for all $j= 1,..,m$.
Conversely, if a kernel is not universal, then there exists a data set $\{x_i,y_i\}_{i=1}^m$ such that for any $\alpha \in \R^m$, there exists some $j\in \{1,\cdots,m\}$ such that $f(y_j) \neq \sum_{i=1}^m \alpha_i k(x_i,x_j)$.
The universality property ensures that the classifier of an SVM designed using a universal kernel will
become increasingly precise as the training data increases, whereas classifiers from a non-universal kernel have limited ability to fit data sets.~\citep[See][]{micchelli_2006}.
	
\begin{defn}[\citet{scholkopf2001generalized}]
A kernel $k:X \times X \rightarrow \R$ is said to be \textbf{universal} on the compact metric space $X$ if it is continuous and there exists an inner-product space $\mcl W$ and feature map, $\Phi : X \rightarrow \mcl W$ such that $k(x,y)=\ip{\Phi(x)}{\Phi(y)}_{\mcl W}$ and where the unique Reproducing Kernel Hilbert Space (RKHS), $\mcl H:=\{f\;:\;f(x)=\ip{v}{\Phi(x)},\; v \in \mcl W\}$
with associated norm $\norm{f}_{\mcl H} :=\inf_{v}\{\norm{v}_{\mcl W}\;:\; f(x)=\ip{v}{\Phi(x)}\}$ is dense in $\mcl C(X):=\{f \,: \, X \rightarrow \R \;:\; $f$\, \text{ is continuous}\}$ where $\norm{f}_{\mcl C}:=\sup_{x\in X} |f(x)|$.
\end{defn}

Recall that the universality property implies the strictly positive definite and characteristic properties on compact domains.
The following definition extends the universal property to a set of kernels.

\begin{defn}
A set of kernel functions $\mcl K$  has the universal property if every kernel function $k \in \mcl K $ is universal.
\end{defn}

%The universal kernel is closely related to strictly positive definite and characteristic kernels.
%\begin{defn}[\cite{sriperumbudur2011universality}]
%A symmetric function $k:X \times X \rightarrow \R$ is said to be \textbf{integrally stricly positive definite} on the compact metric space $X$ if
%\[
%\int_X \int_X k(x, y) f(x) f(y) dx dy > 0 \qquad \text{ for all } f \in \mcl L_2(X).
%\]
%\end{defn}
%The strictly positive definiteness is
%
%\begin{defn}[\cite{sriperumbudur2011universality}]
%A measurable positive kernel $k:X \times X \rightarrow \R$ is said to be \textbf{characteristic} on the compact metric space $X$ if the Hilbert–Schmidt integral operator $K: \mcl L_2(X) \rightarrow \mcl L_2(X)$ is injective
%\[
%(Kf)(x) = \int_X k(x, y) f(y) dy
%\]
%\end{defn}%The following definition extends the universal property to a set of kernels.
%%\begin{defn}
%%A set of kernel functions $\mcl K$ has the universal property if every kernel in the interior of the set, $k \in int(\mcl K)$, is universal.
%%\end{defn}
%
%\begin{thm}[\cite{simon2018kernel}]
%Let $X$ be a compact set and $k:X \times X \rightarrow \R$ be a kernel, then the following statements are equivalent.
%\begin{enumerate}
%\item $k$ is universal
%\item $k$ is characteristic
%\item $k$ is strictly positive definite
%\end{enumerate}
%\end{thm}

\subsection{Density}
The third property of a kernel set, $\mcl K$, is density which ensures that a kernel can be chosen from $\mcl K$ with an associated feature map which optimizes fitting of the data in the associated feature space. This optimality of fit in the feature space may be interpreted differently for SVM and SVR. Specifically, considering SVM for classification, the kernel learning problem determines the kernel $k \in \mcl K$ for which we may obtain the maximum separation in the kernel-associated feature space. According to~\cite{boehmke2019hands}, increasing this separation distance makes the resulting classifier more robust (generalizable). The density property, then, ensures that the resulting kernel learning algorithm will be maximally robust (generalizable) in the sense of separation distance. In the case of SVR, meanwhile, the kernel learning problem finds the kernel $k\in \mcl K$ which permits the ``flattest'' \citep[see ][]{smola2004tutorial} function in feature space. In this case, the density property ensures that the resulting kernel learning algorithm will be maximally robust (generalizable) in the sense of flatness.

Note that the density properties is distinct from the universality property. For instance consider a set containing a single Gaussian kernel function - which is clearly not ideal for kernel learning.  The set containing a single Gaussian is tractable (it has only one element) and every member of the set is universal. However, it is not dense.

These arguments motivate the following definition of the pointwise density property.
\begin{defn}
The set of kernels $\mcl K$ is said to be \textbf{pointwise dense} if for any positive kernel, $k^*$, any set of data $\{x_i\}_{i=1}^m$, and any $\epsilon>0$, there exists $k \in \mcl K$ such that \[
\norm{k(x_i,x_j)-k^*(x_i,x_j)}\le \epsilon.\]
\end{defn}

\section{A General Framework for Representation of Tractable Kernel Sets} \label{sec:GKL}
Having defined three desirable properties of a set of kernels, we now consider a framework designed to facilitate the creation of sets of kernels which meet these criteria. This framework ensures tractability by providing a linear map from positive matrices to positive kernels. This map is defined by a set of bases, $N$. These bases themselves parameterize kernels the image of whose associated integral operators define the feature space. As we will show in Section~\ref{sec:2step}, kernel learning over a set of kernels parameterized in this way can be performed efficiently using a combination of QP and the { Singular Value Decomposition}. Moreover, as we will show in Section~\ref{sec:TK}, suitable choices of $N$ will ensure that the set of kernels has the density and universality properties.

%Here we recall the framework for constructing classes of tractable positive kernel functions as described in~\cite{JMLR}.

\begin{lem}\label{lem:tractable}
Let $N$ be any bounded measurable function $N:  Y \times X \rightarrow \R^{n_P}$ on compact $  X$ and $ Y$. If we define
\begin{equation}
{\mcl K}  := \left\{ k ~\bigg|~ k(x,y)= \int_{Y}  N(z,x)^T P N(z,y) dz, \quad P \succeq 0 \right\}, \label{eqn:tractable} \end{equation}
then any $k\in \mcl K$ is a positive kernel function and $\mcl K$ is tractable.
\end{lem}
\begin{proof}
The proof is straightforward. Given a kernel, $k$, denote the associated integral operator by $I_k$ so that
\[
(I_k \phi)(s)=\int_{X}k(s,t)\phi(t)dt.
\]
If $k\in \mcl K$, it has the form of Eqn.~\eqref{eqn:tractable} for some $P\succeq 0$. Now define $k_{\half}(x,y):=P^{\half}N(x,y)$. Then $I_k=I_{k_{\half}}^*I_{k_{\half}} \succeq  0$ where adjoint is defined with respect to the $L_2$ inner product. This establishes positivity of the kernel. Note that if $I_{k_{\half}}\in \mcl A$ for some *-algebra $\mcl A$, then $I_k\in \mcl A$.

For tractability, we note that for a given $N$, the map $P \mapsto k$ is linear. Specifically,
\[
k(x,y) =\sum\nolimits_{i,j=1}^{n_P}  P_{i,j} G_{i,j}(x,y),\]
where
\begin{equation}
G_{i,j}(x,y) =\int_{Y}N_{i}(z,x)N_{j}(z,y)dz \label{eqn:Gij},
\end{equation}
and thus by Definition~\ref{Tractability}, $\mcl K$ is tractable.
\end{proof}

\noindent \paragraph{Note:} Using the notation for integral operators in the proof of Lemma~\ref{lem:tractable}, we also note that for any $k \in \mcl K$,
\begin{equation}
I_k=\sum\nolimits_{i,j=1}^{n_P} P_{i,j} I_{G_{i,j}}=\sum\nolimits_{i,j=1}^{n_P} P_{i,j} I_{N_{i}}^*I_{N_{j}}.\label{eqn:operator}
\end{equation}

For convenience, we refer to a set of kernels defined as in Eqn.~\eqref{eqn:tractable} as a Generalized Kernel Set, a kernel from such set as a Generalized Kernel, and the associated kernel learning problem in~\eqref{eqn:OPT} as Generalized Kernel Learning (GKL)~\eqref{eqn:tractable}. This is to distinguish such kernels, sets and problems from Tessellated Kernel Learning, which arises from a particular choice of $N$ in the parameterization of $\mcl K$. This distinction is significant, as the algorithms in Sections~\ref{sec:2step} apply to the Generalized Kernel Learning problem, while the results in Section~\ref{sec:TK} only apply to the particular case of Tessellated Kernel Learning.

\section{An Efficient Algorithm for Generalized Kernel Learning in Classification and Regression Problems}\label{sec:2step}
In this section, we assume a family of kernel functions, $\mcl K$, has been parameterized as in~\eqref{eqn:tractable}, and formulate the kernel learning optimization problem for both classification and regression --- representing this as a minimax saddle point problem. This formulation enables a decomposition into convex primal and dual sub-problems, $OPT\_A(P)$ and $OPT\_P(\alpha)$ with no duality gap. We then consider the Frank-Wolfe algorithm and show using Danskin's Theorem that the gradient step can be efficiently computed using the primal and dual sub-problems. Finally, we propose efficient algorithms for computing $OPT\_A(P)$ and $OPT\_P(\alpha)$: in the former case using an efficient Sequential Minimal Optimization (SMO) algorithm for convex QP and in the latter case, using an analytic solution based on the Singular Value Decomposition.

%to calculate the gradient of $OPT\_A(P)$, and, based on ... for any problem of the ...., we show that this algorithm achieves worst-case linear convergence. In Section 6, we will propose modifications to the algorithm which increase this to quadratic convergence using ...

\subsection{Primal-Dual Decomposition}
For convenience, we define the feasible sets for the sub-problems as
\begin{align*}
\mcl X:&=\{P \in \S^{n_P}\;:\; \text{trace}(P) = n_P,\; { P \succeq  0}\},\\
\mcl Y_c:&=\{\alpha \in \R^m\; : \; \sum\nolimits_{i=1}^m \alpha_iy_i = 0,\; 0 \leq \alpha_i \leq C \},\\[-0mm]
\mcl Y_r:&=\{\alpha \in \R^m\;:\; \sum\nolimits_{i=1}^m \alpha_i = 0, \; \alpha_i \in [-C, C]\},
\end{align*}
where $m$ and $C$ are as defined in Optimization Problem~\ref{eqn:OPT}.
In this section, we typically use the generic form $\mcl Y_{\star}$ to refer to either $\mcl Y_c$ or $\mcl Y_r$ depending on whether the algorithm is being applied to the classification or regression problem. To specify the objective function we define $\lambda(\alpha,P)$ as
\begin{equation} \label{OalphaP}
\lambda(\alpha,P):=   -\frac{1}{2} \sum\limits_{i=1}^m\sum\limits_{j=1}^m \alpha_i \alpha_j   \int_Y   N(z,x_i)^T P N(z,y_j) dz,
\end{equation}
where the bases, $N$, and domain, $Y$, are those used to specify the kernel set, $\mcl K$, in Eqn.~\eqref{eqn:tractable}. Additionally, we define $\kappa_c(\alpha):=\sum\nolimits_{i=1}^m \alpha_i$ and
\[
\kappa_r(\alpha):=-\epsilon \sum\nolimits_{i=1}^m |\alpha_i| + \sum\nolimits_{i=1}^m y_i\alpha_i
\]
where, again, we use $\kappa_{\star}=\kappa_c$ for classification and $\kappa_{\star}=\kappa_r$ for regression.

Using the formulation in, e.g.~\cite{lanckriet_2004}, it can be shown that if the family $\mcl K$ is parameterized as in Eqn.~\eqref{eqn:tractable}, then the Generalized Kernel Learning optimization problem in Eqn.~\eqref{eqn:OPT} can be recast as the following minimax saddle point optimization problem.
\begin{align}\label{KernelSVC_P}% \label{OptimalKernelSVM}
OPT_P&:=\min_{P \in \mcl X} \max_{~\alpha \in \mcl Y_{\star}}  \quad \lambda(e_{\star} \odot \alpha,P) + \kappa_{\star}(\alpha),
\end{align}
where $\odot$ indicates elementwise multiplication. For classification, $\mcl Y_\star=\mcl Y_c$, $\kappa_\star=\kappa_c$, and $e_{\star}=e_c: = y$ (vector of labels). For regression, $\mcl Y_\star=\mcl Y_c$, $\kappa_\star=\kappa_r$, and $e_{\star}=e_r: = \mathbf{1}_m$ (vector of ones).

\textbf{Minimax Duality.} To find the dual, $OPT_D$ of the kernel learning optimization problem ($OPT_P$), we formulate two sub-problems:
\begin{equation} \label{opt:AP}
    OPT\_A(P) ~:= \max_{\alpha \in \mcl Y_{\star}} \lambda(e_{\star} \odot \alpha, P)+\kappa_{\star}(\alpha)
\end{equation}
and 
{ 
\begin{equation}
OPT\_P(\alpha):=   \min_{P\in \mcl X} ~~ \lambda(  e_{\star}\odot \alpha, P)+\kappa_{\star}(\alpha) :=   \min_{P\in \mcl X}   \left<D(\alpha), P\right>,\label{opt:OPTP}
\end{equation}
where
\begin{equation}
    D_{i,j}(\alpha) = \sum\nolimits_{k,l = 1}^m (\alpha_k y_k) G_{i,j}(x_k,x_l) (\alpha_l y_l) \label{eqn:Dalpha} %\\[-1mm]
%    G_{i,j}(x,y) &:= \begin{cases} g_{i,j}(x,y) & \text{if}~ i \leq \frac{q}{2}, j \leq \frac{q}{2} \\ t_{i,j}(x,y) & \text{if}~ i \leq \frac{q}{2}, j > \frac{q}{2} \\ t_{i,j}(y,x) & \text{if} ~i > \frac{q}{2}, j \leq \frac{q}{2} \\ h_{i,j}(x,y) & \text{if}~ i > \frac{q}{2}, j > \frac{q}{2} \end{cases} \notag
\end{equation}
and the $G_{i,j}(x, y)$ are as defined in~\eqref{eqn:Gij}.}
%\begin{align}
%OPT\_{P}(\alpha):=\min_{P\in \mcl X} ~~ \lambda(  e_{\star}\odot \alpha, P)+\kappa_{\star}(\alpha).\label{opt:OPTP}
%\end{align}
Now, we have that
\begin{align*}
OPT_P=\min_{P \in \mcl X} ~ OPT\_A(P)
\end{align*}
and its minmax dual is
\begin{align}  \label{KernelSVC_D}
OPT_D&=\max_{\alpha \in \mcl Y_{\star}} OPT\_P( \alpha) =\max_{~\alpha \in \mcl Y_{\star}} \min_{P \in \mcl X} \quad \lambda(e_{\star} \odot \alpha,P) + \kappa_{\star}(\alpha). \nonumber
\end{align}
The following lemma states that there is no duality gap between $OPT_P$ and $OPT_D$ - a property we will use in our termination criterion.
\begin{lem}$OPT_P=OPT_D$. Furthermore, $\{\alpha^*, P^*\}$ solve $OPT_P$ if and only if $OPT\_P( \alpha^*)=OPT\_A(P^*)$.
\end{lem}
\begin{proof}

% \hl{ FROM SUPPLEMENTARY}

For any minmax optimization problem with objective function $\phi$, we have
\[d^* = \max_{\alpha \in \mcl Y} \min_{P \in \mcl X} \phi(P,\alpha) \leq   \min_{P \in \mcl X}\max_{\alpha \in \mcl Y}  \phi(P,\alpha) = p^* \]
and strong duality holds ($p^*-d^* = 0$) if $\mcl X$ and $\mcl Y$ are both convex and one is compact, $\phi(\cdot,\alpha)$ is convex for every $\alpha \in \mcl Y$ and $\phi(P,\cdot)$ is concave for every $P \in \mcl X$, and the function $\phi$ is continuous~\citep[See][]{fan1953minimax}. In our case, these conditions hold for both classification ($\phi(P,\alpha)=\lambda (\alpha\odot y, P)+\kappa_c(\alpha)$) and regression ($\phi(P,\alpha)=\lambda(\alpha, P)+\kappa_r(\alpha))$. Hence $OPT_P=OPT_D$. Furthermore, if $\{\alpha^*, P^*\}$ solve $OPT_P$ then
\begin{align*}
OPT\_P( \alpha^*)&=\max_{\alpha \in \mcl Y} OPT\_P( \alpha) =\min_{P \in \mcl X} OPT\_A(P) =OPT\_A(P^*).
\end{align*}
Conversely, suppose $\alpha \in \mcl Y$, $P \in \mcl X$, then
\begin{align*}\small
OPT\_P(\alpha) \le \max_{\alpha \in \mcl Y}\; OPT\_P(\alpha) &= OPT\_P( \alpha^*) \\
&= OPT\_A(P^*) = \min_{P \in \mcl X} OPT\_A(P) \le OPT\_A(P).
\end{align*}
Hence if $OPT\_A(P)=OPT\_P(\alpha)$, then $OPT\_A(P)=OPT\_A(P^*)=OPT\_P(\alpha^*)=OPT\_P(\alpha)$ and hence $P$ and $\alpha$ solve $OPT\_A$ and $OPT\_P$, respectively.
\end{proof}
%\begin{proof}
%For any minmax optimization problem with objective function $\phi$, we have
%\[d^* = \max_{\alpha \in \mcl Y} \min_{P \in \mcl X} \phi(P,\alpha) \leq   \min_{P \in \mcl X}\max_{\alpha \in \mcl Y}  \phi(P,\alpha) = p^*, \]
%and strong duality holds ($p^*-d^* = 0$) if $\mcl X$ and $\mcl Y$ are both convex and one is compact, $\phi(\cdot,\alpha)$ is convex for every $\alpha \in \mcl Y$ and $\phi(P,\cdot)$ is concave for every $P \in \mcl X$, and the function $\phi$ is continuous \cite{fan1953minimax}. In our case, these conditions hold for both classification and regression where $\phi(P,\alpha)=O(\alpha, P)+\kappa_r(\alpha)\; \text{or}\;O(\alpha\odot y, P)+\kappa_c(\alpha)$. Hence $OPT_P=OPT_D$. Furthermore, if $\{\alpha^*, P^*\}$ solve $OPT_P$ then\vspace{-1.5mm}
%\begin{align*}
%OPT\_P( \alpha^*)&=\max_{\alpha \in \mcl Y} OPT\_P( \alpha)\\
%&=\min_{P \in \mcl X} OPT\_A(P) =OPT\_A(P^*).
%\end{align*}
%Conversely, suppose $\alpha \in \mcl Y$, $P \in \mcl X$, then \vspace{1mm}
%\begin{align*}\small
%OPT\_P(\alpha) &\le \max_{\alpha \in \mcl Y}\; OPT\_P(\alpha) \\
%&= OPT\_P( \alpha^*) = OPT\_A(P^*) \\
%&= \min_{P \in \mcl X} OPT\_A(P) \le OPT\_A(P).
%\end{align*}
%Hence if $OPT\_A(P)=OPT\_P(\alpha)$, then $OPT\_A(P)=OPT\_A(P^*)=OPT\_P(\alpha^*)=OPT\_P(\alpha)$ and hence $P$ and $\alpha$ solve $OPT\_A$ and $OPT\_P$, respectively.
%\end{proof}
Finally, we show that $OPT\_A(P)$ is convex with respect to $P$ - a property we will use in Thm.~\ref{thm:convergence}.
\begin{lem} \label{thm:convexity}
Let $OPT\_A(P)$ be as defined in~\eqref{opt:AP}.  Then, the function $OPT\_A(P)$ is convex with respect to $P$.
\end{lem}
\begin{proof}
 
 It is a well known property of convex functions (e.g.~\cite{bertsekas1998nonlinear}) that if $\mcl Y_*$ is a compact set and $\phi(\alpha, P)$ is convex with respect to $P \in \mcl X$ for every $\alpha \in \mcl Y_*$, then if $g(P) = \max_{\alpha \in \mcl Y_*} \phi(\alpha, P)$ exists for every $P \in \mcl X$, then $g(P)$ is convex with respect to $P \in \mcl X$. These conditions are readily verified using the definition of $OPT\_A(P)$ for both classification and regression, where $\phi(\alpha, P)= \lambda(e_{\star} \odot \alpha, P)+\kappa_{\star}(\alpha)$ and $g(P)=OPT\_A(P)$. 
 %\hl{ SUPPLEMENTARY}
%
%For simplicity, let $e_{\star}=e_r \in \R^m$ which is the vector of ones and hence $e_{\star} \odot \alpha = \alpha$.  The function $OPT\_A(P): \mcl X \rightarrow \R$ is convex if and only if for any $X,V \in \mcl X$, $g(t) := OPT\_A(X + tV)$ is convex in $t$ for all $t\in H:=\{s\in \R \,|\, X + s V \succeq 0\}$.  To prove convexity of $g(t)$ we must prove that,
%\[g(\theta t_1 + (1-\theta) t_2) \leq \theta g(t_1) + (1-\theta) g(t_2)\]
%for any $\theta \in [0,1]$ and $t_1,t_2 \in H$.
%Since $t_1,t_2\in H$, by definition if we let $P_1 = X+t_1V$ and $P_2 = X+t_2V$, then $P_1,P_2$ are positive semi-definite matrices. Now, since $\lambda(\alpha,P)$ is linear with respect to $P$  we have that,
%\begin{align*}
%   g(\theta t_1 + (1  -  \theta) t_2) &= \max_{\alpha \in \mcl Y_{\star}} \lambda(\alpha, X+(\theta t_1 + (1  -  \theta) t_2)V)+\kappa_{\star}(\alpha)\\
%   &=\max_{~\alpha \in \mcl Y_{\star}} \lambda(\alpha,\theta P_1    +
%   (1  -  \theta) P_2 )   +   \kappa_{\star}(\alpha)\\
%    &= \max_{~\alpha \in \mcl Y_{\star}} \theta\left(\lambda(\alpha,P_1)+\kappa_{\star}(\alpha) \right) +   (1  -  \theta) \left( \lambda(\alpha,P_2)   +  \kappa_{\star}(\alpha) \right)\\
%    & \leq \max_{~\alpha \in \mcl Y_{\star}} \theta\left(\lambda(\alpha,P_1)+\kappa_{\star}(\alpha) \right) +\max_{~\alpha \in \mcl Y_{\star}} (1-\theta) \left(\lambda(\alpha,P_2) +\kappa_{\star}(\alpha) \right) \\
%    & \leq \theta g(t_1) + (1-\theta) g(t_2).\\[-9mm]
%\end{align*}
%The proof similarly holds for $e_{\star}=e_c$.
\end{proof}

\begin{figure}[!t]
 \begin{minipage}[t]{0.5\textwidth}
\begin{algorithm}[H]\hspace{-7mm}
\begin{algorithmic}
{\small
\STATE \texttt{Initialize} $P_0$ as any point in $\mcl X$.;
\STATE  \texttt{Step 1:}
\STATE ~~ {  $S_k = \arg\min_{S \in \mcl{X}} \left< \nabla f(P_k), S \right> $}
\STATE  \texttt{Step 2:}
\STATE ~~ $\displaystyle \gamma_k = \arg \min_{\substack{\gamma\in [0,1]}} f(P_k+\gamma(S_k-P_k))$ \\
\STATE \texttt{Step 3:}
\STATE ~~ $P_{k+1}= P_k + \gamma_k \left( S_k - P_k \right), k = k + 1,$\\
\STATE Return to step 1 unless stopping criteria is met.}
\end{algorithmic}
\caption{The Frank-Wolfe Algorithm for Matrices.} \label{alg:FW}
\end{algorithm}
\end{minipage}
\hfill
\begin{minipage}[t]{0.5\textwidth}
\begin{algorithm}[H]\hspace{-7mm}
\begin{algorithmic}
{\small
%\STATE \texttt{Given} $t \in [0,1]$;
\STATE \texttt{Initialize} $P_0=I$, $k=0$, $\alpha_0 = OPT\_A(P_0)$;
\STATE \texttt{Step 1a:}~~ $\alpha_k = \arg OPT\_A(P_k)$
\STATE \texttt{Step 1b:}~~ $S_k = \arg OPT\_P(\alpha_k)$
\STATE \texttt{Step 2:}
\STATE~~  $\displaystyle \gamma_k = \arg \min_{\substack{\gamma\in [0,1]}} OPT\_A(P_k+\gamma(S_k-P_k))$ \\
\STATE \texttt{Step 3:}
\STATE~~  $P_{k+1} = P_k + \gamma_k(S_k - P_k)$, $k = k + 1$
\STATE Return to step 1 unless $OPT\_P(\alpha_k) - OPT\_A(P_k) < \epsilon.$
}
\end{algorithmic}
\caption{Proposed FW Algorithm for GKL.} \label{FWTKL}
\end{algorithm}
\end{minipage}

\end{figure}

\subsection{Primal-Dual Frank-Wolfe Algorithm}
For an optimization problem of the form
\begin{align*}
\min_{S \in \mcl X} f(S),
\end{align*}
where $\mcl X$ is a convex subset of matrices and $\left<\cdot,\cdot\right>$ is the Frobenius matrix inner product, the Frank-Wolfe (FW) algorithm~\citep[See, e.g.][]{frank1956algorithm} may be defined as in Algorithm~\ref{alg:FW}.

In our case, we have $f(Q)=OPT\_A(Q)$ so that
\begin{align*}
OPT_P=\min_{P \in \mcl X} OPT\_A(P).
\end{align*}
{  Implementation of the FW algorithm requires us to compute $\nabla OPT\_A(P_k)$ at each iteration.
To address this issue, we propose a way to efficiently compute the sub-problems $OPT\_A$ and $OPT\_P$, as shown in Subsections~\ref{subsec:step1} and~\ref{subsec:step2}.} Furthermore, in Lemma~\ref{thm:Derivative}, we will show that these sub-problems can be used to efficiently compute the gradient { $\nabla OPT\_A(P_k)$} - allowing for an efficient implementation of the FW algorithm. {  Lemma~\ref{thm:Derivative} uses a variation of Danskin's theorem~\citep[generalized in][]{bertsekas1998nonlinear}.}
\begin{prop}[Danskin's Theorem] \label{thm:Danskins}
Let ${\mcl Y \subset \R^m}$ be a compact set, and let $\phi : \mcl X \times \mcl Y \rightarrow \R$ be continuous such that $\phi(\cdot,\alpha): \mcl X \rightarrow \R$ is convex for each $\alpha \in \mcl Y$.  Then for $P\in \mcl X$, if
\begin{align*}
   {\mcl Y_0}(P) = \left\{ \bar{\alpha} \in \mcl Y~~\big|~~\phi(P,\bar{\alpha}) = \max_{\alpha \in {\mcl Y}} \phi(P,\alpha) \right\} \\[-8mm]
\end{align*}
consists of only one unique point, $\bar{\alpha}$, and $\phi(\cdot,\bar{\alpha})$ is differentiable at $P$ then the function { $f(Q) = \max_{\alpha \in \mcl Y} \phi(Q,\alpha)$} is differentiable at $P$ and
\[  
\nabla f(P) = \nabla_Q \phi(Q,\bar{\alpha})\big|_{Q = P}.
\]
%where { $\nabla_Q \phi(Q,\bar{\alpha})$} is the vector with coordinates
%\[ 
% \frac{\partial \phi(Q,\bar{\alpha})}{\partial Q_i}, \quad i = 1,...,n.\]
\end{prop}
{  Prop.~\ref{thm:Danskins} can now be used to prove the following.}
\begin{lem} \label{thm:Derivative}
If $OPT\_A$ and $OPT\_P$ are as defined in Eqns.~\eqref{opt:AP} and~\eqref{opt:OPTP}, then for any $P_k\succ 0$, we have
\begin{align*}
&\arg \min_{S \in \mcl{X}}\left< \nabla OPT\_A({ P_k}), S \right>  = \arg OPT\_P(\arg OPT\_A(P_k)). \\[-9mm]
\end{align*}
\end{lem}
\begin{proof}
%\[
%OPT\_{P}(\alpha):=\min_{P\in \mcl X} ~~ O(e_{\star}\odot \alpha , P)+\kappa_{\star}(\alpha)
%\]
%\begin{align}
%    OPT\_A(P) ~:= \max_{\alpha \in \mcl Y_{\star}} O(e_{\star} \odot \alpha, P)+\kappa_{\star}(\alpha).
%\end{align}
%\begin{equation}
%O(\alpha,P):=   -\frac{1}{2} \sum\limits_{i=1}^m\sum\limits_{j=1}^m \alpha_i \alpha_j   \int_{a}^b   N_T^d(z,x_i)^T P N^d_{T}(z,y_j) dz,
%\end{equation}
For simplicity, we utilize the definition of $D(\alpha)$ which will be given in Eqn.~\eqref{eqn:Dalpha} so that $\lambda(e_{\star} \odot \alpha, P):=\left<D(\alpha), P\right>$. Now, since $\lambda(\alpha,P)$ is {  strictly concave} in $\alpha$, for any $P_k\succ 0$, $OPT\_A(P_k)$ has a unique solution and hence we have by Danskin's Theorem that
\begin{align*}
    &\arg \min_{S \in \mcl{X}}{\big<} \nabla OPT\_A({ P_k}) , S {\big>}\\[-2mm]
     &=\arg \min_{S \in \mcl{X}}{\big<} \nabla_Q \left[\max_{\alpha \in \mcl Y_{\star}} \left(\left<D(\alpha), Q\right>+\kappa_{\star}(\alpha)\right)\right]_{Q=P_k}, S {\big>} \\
     &=\arg \min_{S \in \mcl{X}}{\big<} \nabla_Q \left[\left<D(\bar \alpha), Q\right>+\kappa_{\star}(\bar \alpha)\right]_{Q=P_k}, S {\big>} \\[-6mm]
\end{align*}
where $\bar \alpha = \arg OPT\_A(P_k)$. Hence,
\begin{align*}
&\arg \min_{S \in \mcl{X}}{\big<} \nabla_Q \left[\left<D(\bar\alpha), Q\right>+\kappa_{\star}(\bar \alpha)\right]_{Q=P_k}, S {\big>}
=\arg \min_{S \in \mcl{X}}{\big<} \nabla_Q \left[\left<D(\bar\alpha), Q\right>\right]_{Q=P_k}, S {\big>} \\
%     &=\min_{S \in \mcl{X}}{\big<} \nabla_Q \left[\left<C(\bar\alpha), Q\right>\right]_{Q=P_k}, S {\big>} \\
     &=\arg \min_{S \in \mcl{X}}{\big<} D(\bar\alpha),  S {\big>}
      = \arg OPT\_P(\bar \alpha)
        = \arg OPT\_P(\arg OPT\_A(P_k)).\\[-11mm]
\end{align*}
\end{proof}

We now propose an efficient implementation of the FW GKL algorithm, as defined in Algorithm~\ref{FWTKL}, based on efficient algorithms for computing $OPT\_A$ and $OPT\_P$ as will be defined in Subsections~\ref{subsec:step1} and~\ref{subsec:step2}.

In the following theorem, we use convergence properties of the FW GKL algorithm to show that Algorithm~\ref{FWTKL} has { worst-case $O\left(\frac{1}{k}\right)$ convergence}. Note that when higher accuracy is required, we may utilize recently proposed primal-dual algorithms for { $O\left(\frac{1}{k^2}\right)$ convergence}, as will be discussed in Section~\ref{sec:APD}.

\begin{thm} \label{thm:convergence}
Algorithm~\ref{FWTKL} returns iterates $P_k$ and $\alpha_k$ such that,
$|\lambda(\alpha_k,P_k)+\kappa_{\star}(\alpha_k) - OPT_P| < O(\frac{1}{k})$.
%where, \vspace{-2.5mm}
%\[P^* := \arg\underset{P \in \mcl X}{\min} OPT\_A(P)
%\]
\end{thm}
\begin{proof}
If we define $f=OPT\_A$, then Lemma~\ref{thm:Derivative} shows that $f$ is differentiable and, if the $P_k$ satisfy Algorithm~\ref{FWTKL}, that the $P_k$ also satisfy Algorithm~\ref{alg:FW}. In addition, Lemma~\ref{thm:convexity} shows that $f(Q)=OPT\_A(Q)$ is convex in $Q$. It has been shown in, e.g.~\cite{jaggi2013revisiting}, that if $\mcl X$ is convex and compact and $f(Q)$ is convex and differentiable on $Q \in \mcl X$, then the FW Algorithm produces iterates $P_k$, such that,
$f(P_k) - f(P^*) < O(\frac{1}{k})$ where
\[
f(P^*)=  \min_{P \in \mcl X} f(P)=  \min_{P \in \mcl X} OPT\_A(P)=OPT_P.
\]
%where the $Q$ is convex and compact and $h(\mu)$ is convex and differentiable on $Q$, then by~\cite{jaggi2013revisiting}, the Frank-Wolfe Algorithm produces iterates $\mu$, such that,
%$h(\lambda_k) - h(\lambda^*) < O(\frac{1}{k})$.
Finally, we note that
\begin{align*}
\lambda(\alpha_k,P_{k})+\kappa_{\star}(\alpha_k)
 =&\lambda(\arg OPT\_A(P_{k}),P_{k})+\kappa_{\star}(\arg OPT\_A(P_{k}))\\
&=\max_{\alpha\in \mcl Y_{\star}} \lambda(\alpha,P_k)+\kappa_{\star}(\alpha)=OPT\_A(P_{k})=f(P_k),
%
%f(S_k)&=OPT\_A(S_k)=OPT\_A(OPT\_P(\alpha_k))\\
%&=\max_\alpha \lambda(\alpha,S_k)+\kappa_{\star}(\alpha)=
\end{align*}
which completes the proof.
%By Thm.~\ref{thm:convexity}, $OPT\_A(P)$ is convex with respect to $P$, and differentiable on $Q$ by Lemma~\ref{thm:Derivative}.  The set of positive definite matrices is convex and compact hence Algorithm~\ref{FWTKL} returns iterates $P_k$ such that,
%$OPT\_A(P_k) - OPT\_A(P^*) < O(\frac{1}{k})$.
\end{proof}
 In the following subsections, we provide efficient algorithms for computing the sub-problems $OPT\_A$ and $OPT\_P$.

\subsection{Step 1, Part A: Solving $OPT\_A(P)$ } \label{subsec:step1}
For a given { $P \succeq 0$}, $OPT\_A(P)$ is a convex Quadratic Program (QP). General purpose QP solvers have a worst-case complexity which scales as $O(m^3)$ (See~\cite{ye1989extension}) where, when applied to $OPT\_A$, $m$ becomes the number of samples. This computational complexity may be improved, however, by noting that $OPT\_A$ is compatible with the representation defined in~\cite{LibSVM} for QPs derived from support vector machine problems. In this case, the algorithm in LibSVM reduces the computational burden somewhat. This improved performance is illustrated in Figure~\ref{fig:kernel_complexity} where we observe the achieved complexity scales as $O(m^{2.3})$. Note that for the 2-step algorithm proposed in this manuscript, solving the QP in $OPT\_A(P)$ is significantly slower that solving the Singular Value Decomposition required for $OPT\_P(\alpha)$, which is defined in the following subsection. %Of course, this achieved complexity of $O(m^{2.1})$ is also significantly faster than solving the large SDP, as described in~\cite{lanckriet_2004},~\cite{qiu2005multiple}, and~\cite{JMLR}. This complexity comparison will be further discussed in Section~\ref{sec:scalability}.

\subsection{Step 1, Part B: Solving $OPT\_P(\alpha)$} \label{subsec:step2}
For a given $\alpha$, $OPT\_P(\alpha)$ is an SDP. Fortunately, however, this SDP is structured so as to admit an analytic solution using the { Singular Value Decomposition (SVD)}.
{ To solve $OPT\_P(\alpha)$ we minimize $\lambda(e_{\star} \odot \alpha,P)$ from Eq.~\eqref{OalphaP}
which %, {\color{red} as per Corollary~\ref{corr},} 
is linear in $P$ and can be formulated as}
\begin{align*}
OPT\_P(\alpha):=   \min_{\substack{P \in \S^{n_P}\\ \text{trace}(P) = n_P \\ { P \succeq 0}}}   \lambda(e_{\star} \odot \alpha, P) :=   \min_{\substack{P \in \S^{n_P}\\ \text{trace}(P) = n_P \\ { P \succeq 0}}}   \left<D(\alpha), P\right>,  \notag  \\[-10mm]
\end{align*}
where
\begin{equation*}
    D_{i,j}(\alpha) = \sum\nolimits_{k,l = 1}^m (\alpha_k y_k) G_{i,j}(x_k,x_l) (\alpha_l y_l)  %\\[-1mm]
%    G_{i,j}(x,y) &:= \begin{cases} g_{i,j}(x,y) & \text{if}~ i \leq \frac{q}{2}, j \leq \frac{q}{2} \\ t_{i,j}(x,y) & \text{if}~ i \leq \frac{q}{2}, j > \frac{q}{2} \\ t_{i,j}(y,x) & \text{if} ~i > \frac{q}{2}, j \leq \frac{q}{2} \\ h_{i,j}(x,y) & \text{if}~ i > \frac{q}{2}, j > \frac{q}{2} \end{cases} \notag
\end{equation*}
and the $G_{i,j}(x, y)$ are as defined in~\eqref{eqn:Gij}.

%, $t$ and $h$ can be found in Corollary~\ref{corr}.

The following theorem gives an analytic solution for $OPT\_P$ using the SVD.
\begin{thm}
For a given $\alpha$, denote $D_\alpha:=D(\alpha) \in \S^{n_P}$ where $D(\alpha)$ is as defined in Eqn.~\eqref{eqn:Dalpha} and let $D_\alpha=V\Sigma V^T$ be its SVD. Let $v$ be the right singular vector corresponding to the minimum singular value of $D_\alpha$. Then $P^* =n_P vv^T$ solves $OPT\_P(\alpha)$.
\end{thm}
\begin{proof}
Recall $OPT\_P(\alpha)$ has the form
\[\underset{P \in \S^{n_P}}{\min}  \left<D_\alpha, P\right>~~\text{s.t.}~P\succeq 0,~\text{trace}(P) = n_P.
\]
Denote the minimum singular value of $D_\alpha$ as $\sigma_{\min}(D_\alpha)$.  Then for any feasible $P\in \mcl X$, by~\cite{fang1994inequalities} we have
\begin{align*}
\left<D_\alpha, P\right> \geq \sigma_{\min}(D_\alpha) \text{trace}(P) = \sigma_{\min}(D_\alpha) n_P.
\end{align*}
Now consider $P = n_P vv^T \in \S^{n_P}$. $P$ is feasible since $P\succeq 0$, and $\text{trace}(P) = n_P$. Furthermore,
\begin{align*}
\left<D_\alpha, P\right>&=n_P\,\text{trace}(V\Sigma V^T vv^T)=n_P\,\text{trace}(v^TV \Sigma V^T v)\\[-1mm] &=n_P\,\sigma_{\min}(D_\alpha)
\end{align*}
as desired.
%Since $V$ is unitary we have that $\text{trace}(G) = \text{trace}(P)$ and $G \succeq 0$ since $P\succeq 0$. If $G_{1,1} = q$, then $\text{trace}(DG) = q d_1$ which is the minimum possible value of $\text{trace}(C^TP).$  Therefore let $v_1$ be the first eigenvector in $V$ and, $P^* = V G V^T = q v_1 v_1^T$.
\end{proof}
Note that the size, $n_P$, of $D_\alpha$ in $OPT\_P(\alpha)$ scales with the number of features, but not the number of samples ($m$). As a result, we observe that the $OPT\_P$ step of Algorithm~\ref{FWTKL} is significantly faster than the $OPT\_A$ step for large data sets.

%\begin{figure}[t]
%
% \begin{center}
%\includegraphics[trim= 10 0 20 20, clip, %width=.4\textwidth]{Convergence_Cancer_Pima.eps}
%  \end{center}
%        \caption{The normalized objective value error for 1000 iterations of Algorithm~\ref{FWTKL}, on two different classification data sets when compared to the optimal objective function value calculated using the SDP method in~\cite{JMLR}.} \label{fig:convergence}
%  \end{figure}
\subsection{An Accelerated Algorithm for $O\left(\frac{1}{k^2}\right)$ Convergence}\label{sec:APD}
The Frank-Wolfe (FW) GKL algorithm proposed in Section~\ref{sec:2step} has provable {  sublinear convergence $O\left(\frac{1}{k}\right)$}.  While we observe in practice that achieved convergence rates of FW initially exceed this provable bound, when the number of iterations is large (e.g. when duality gap $<10^{-5}$ is desired), the FW algorithm tends to return to {  sublinear convergence} (See Fig.~\ref{fig:apdconvergence}). While {  $O\left(\frac{1}{k}\right)$ convergence} is adequate for most problems, occasionally we may require highly accurate solutions. In such cases, we may look for saddle-point algorithms with {  $O\left(\frac{1}{k^2}\right)$ convergence}, so as to reduce the overall computation time. One such Accelerated Primal Dual (APD) algorithm was recently proposed in~\cite{Hamedani2020APA} and the QP/SVD approach proposed in Subsections~\ref{subsec:step1} and~\ref{subsec:step2} can also be used to implement this algorithm. Specifically, we find that when the QP/SVD approach is applied to APD algorithm, the result is reduced convergence rates for the first few iterations, but improved convergence rates at subsequent iterations. Furthermore, while the per-iteration computational complexity increases with the use of APD, the scaling with respect to feature and sample size remains essentially the same -- See Fig.~\ref{fig:apdconvergence}. However, because most numerical tests in this paper were performed to a primal-dual gap of $10^{-5}$, the APD implementation was not used to produce the results in Section~\ref{sec:8} and hence details are not included. Please see Appendix~\ref{appendix:APD} for additional details.

\section{Tessellated Kernels: Tractable, Dense and Universal} \label{sec:TK}
In this section, we examine the family of kernels defined as in~\eqref{eqn:KernelSet} for a particular choice of $N$. Specifically, let $Y=X=\R^n$ and $Z_d: \R^n  \times \R^n \rightarrow \R^{{d+2n+1 \choose d}}$ be the vector of monomials of degree $d$ or less and define the indicator function for the positive orthant, $\mathbf I: \R^n \rightarrow \R$ as follows.
\[
\mathbf I(z) = \begin{cases}
 1       & \quad z \ge 0\\
 0  & \quad \text{otherwise,}
\end{cases}
\]
where recall $z \ge 0$ if $z_i \ge 0$ for all $i$. We now specify the $N$ which defines $\mcl K$ in~\eqref{eqn:tractable} as $ N^d_{T}: Y \times X \rightarrow \R^{2{d+2n+1 \choose d}}$ for $d \in \N$ as
\begin{align}
N^d_{T}(z,x) = \bmat{Z_d(z,x)\mathbf I(z-x) \\ Z_d(z,x)\left(1 - \mathbf I(z-x) \right) } . \label{eqn:N}
\end{align}
 This assignment $N \rightarrow N^d_T$  defines an associated families of kernel functions, denoted $\mcl K^d_T$ where
\begin{equation}
    \mcl K^d_T:=   \left\{k: k(x,y)   =   \int_{Y}  N_T^d(z,x)^T P  N^d_{T}(z,y) dz, \;\;P\succeq 0,\right\}. \label{eqn:KernelSet}
\end{equation}
The union of such families is denoted $\mcl K_T:=\{k\;:\; k\in \mcl K_T^d,\; d \in \N\}$.

Note that since $Z_d(x,y)$ consists of monomials, it is separable and hence has the form $Z_d(x,y) = Z_{d,a}(x)Z_{d,b}(y)$. This implies that $I_{P^{\half} N^d_{T}}$ (as defined in Eqn.~\eqref{eqn:operator}) has the form given in Eqn.~\eqref{eqn:semisep}. It can be shown that this class of operators forms a *-algebra and hence any kernel in $\mcl K_T$ is semiseparable (extending this term to cover $n$-dimensions) --- implying that for any $k\in \mcl K^d_T$, $I_k$ has the form in Eqn.~\eqref{eqn:semisep}.

In~\cite{JMLR}, this class of kernels was termed ``Tessellated'' in the sense that each datapoint defines a vertex which bisects each dimension of the domain of the resulting classifier/predictor - resulting in a tessellated partition of the feature space.

\subsection{$\mcl K_T$ is Tractable}
The class of Tessellated kernels  is prima facie in the form of Eqn.~\eqref{eqn:tractable} in Lemma~\ref{lem:tractable} and hence is tractable. However, we will expand on this result by specifying the basis for the set of Tessellated kernels, which will then be used in combination with the results of Section~\ref{sec:2step} to construct an efficient algorithm for kernel learning using Tessellated kernels.

\begin{corr}\label{corr}
Suppose that $a<b \in \R^n$, and $d \in\N$. We define the finite set $D_d:=\{(\delta,\lambda)\in \N^{n}\times \N^{n} : \norm{(\delta,\lambda)}_1 \le d\}$.  Let $\{ [\delta_i, \gamma_i] \}_{i=1}^{{ \frac{1}{2}n_p}} \subseteq D_d$ be some ordering of $D_d$ where { $n_p=2{d+2n+1 \choose d}$}. Define $Z_d(x,z)_j = x^{\delta_j} z^{\gamma_j}$ where $x^{\delta_j} z^{\gamma_j}:=\prod_{i=1}^n x_i^{\delta_{j,i}}z_i^{\gamma_{j,i}}$. Now let $k$ be as defined in Eqn.~\eqref{eqn:tractable} for some { $P\succeq 0$} and where $N$ is as defined in Eqn.~\eqref{eqn:N}. %If we partition $P = \bmat{Q & R\\ R^T & S}$ then we have,
Then we have 
\begin{align*}
    k(x,y)=\sum\nolimits_{i,j=1}^{{ n_P}}  P_{i,j} G_{i,j}(x,y),%\sum\nolimits_{i,j=1}^q &Q_{i,j} g_{i,j}(x,y) + R_{i,j}t_{i,j}(x,y) +R^T_{i,j}t_{i,j}(y,x)  + S_{i,j}  h_{i,j}(x,y) \nonumber\\[-7mm]\notag
\end{align*}
where
\begin{align*}
    G_{i,j}(x,y) &:= \begin{cases}
g_{i,j}(x,y) & \text{if}~ i \leq \frac{n_P}{2}, j \leq \frac{n_P}{2} \\
 t_{i,j}(x,y) & \text{if}~ i \leq \frac{n_P}{2}, j > \frac{n_P}{2} \\
 t_{i,j}(y,x) & \text{if} ~i > \frac{n_P}{2}, j \leq \frac{n_P}{2} \\
h_{i,j}(x,y) & \text{if}~ i > \frac{n_P}{2}, j > \frac{n_P}{2}
\end{cases}
\end{align*}
and
where $g_{i,j},t_{i,j},h_{i,j}:\R^{2n}\rightarrow \R$ are defined as
\begin{align}\label{g}
g_{i,j}(x,y) &:= x^{\delta_i}y^{\delta_j} T(p^*(x,y),b,\gamma_{i} + \gamma_{j} + \mathbf{1} ), \notag\\ t_{i,j}(x,y) &:= x^{\delta_{i}}y^{\delta_{j}} T(x,b,\gamma_{i} + \gamma_{j} + \mathbf{1}  ) - g_{i,j}(x,y),   \\
 h_{i,j}(x,y) &:= x^{\delta_{i}}y^{\delta_{j}} T(a,b,\gamma_{i} + \gamma_{j} + \mathbf{1}  ) - g_{i,j}(x,y)    -t_{i,j}(x,y) - t_{i,j}(y,x), \notag\\[-7mm]\notag
\end{align}
where $\mathbf{1} \in \N^n$ is the vector of ones, $p^*:\R^{2n}\rightarrow \R^n$ is defined elementwise as
$p^*(x,y)_i = \max \{x_i,y_i \}$, and $T:\R^n \times \R^n \times \N^n\rightarrow \R$ is defined as
\begin{align*}
T(x,y,\zeta) = \prod\nolimits_{j=1}^n \left(  \frac{y_j^{\zeta_j}}{\zeta_j}-\frac{x_j^{\zeta_j}}{\zeta_j}\right). \\[-10mm]
\end{align*}
%{
%\color{red} Furthermore, for $G_{i, j}(x, y)$ defined in~\eqref{eqn:Gij}, we have,
%\begin{align*}
%    G_{i,j}(x,y) &:= \begin{cases} g_{i,j}(x,y) & \text{if}~ i \leq \frac{q}{2}, j \leq \frac{q}{2} \\ t_{i,j}(x,y) & \text{if}~ i \leq \frac{q}{2}, j > \frac{q}{2} \\ t_{i,j}(y,x) & \text{if} ~i > \frac{q}{2}, j \leq \frac{q}{2} \\ h_{i,j}(x,y) & \text{if}~ i > \frac{q}{2}, j > \frac{q}{2} \end{cases},
%\end{align*}
%}
\end{corr}
The proof of Corollary~\ref{corr} can be found in~\cite{JMLR}.

\subsection{$\mcl K_T$ is Dense}

As per the following Lemma from~\cite{JMLR}, the set of Tessellated kernels satisfies the pointwise density property.
\begin{thm} \label{thm:TessellatedOptimal}
For any positive semidefinite kernel matrix $K^*$ and any finite set $\{x_i\}_{i=1}^m$, there exists a $d \in\N$ and $k \in\mcl K_T^d$ such that {$K^*_{i,j} = k(x_i,x_j)$  for all $i,j$}.
\end{thm}

\subsection{$\mcl K_T$ is Universal}
To show that $\mcl K_T$ is universal, we first show that the auxiliary kernel $k(x, y) = \int_{  Y} \mbf I(z- x) \mbf I(z -y)dz$ is universal.
\begin{lem}\label{lem:universal}
For any $a,b,\delta \in \R^n$ with $a<b$ and $\delta>0$, let $Y = [a - \delta,  b + \delta]$ and $ X = [a, b]$. Then the kernel
\begin{align*}
k(x, y) &= \int_{  Y} \mbf I(z- x) \mbf I(z -y)dz  = \int_{\substack{x \leq z \\ y \leq z  \\ z \leq b+\delta\\}} 1 dz   = \prod_{i=1}^n (b_i+\delta_i - \max{\{x_i, y_i\}})= \prod_{i=1}^n k_i(x_i,y_i)
\end{align*}
%\[
%k(x, y) = \int_{  Y} \mbf I(z- x) \mbf I(z -y)dz = \prod_{i=1}^N (b_i+\delta_i - \max{\{x_i, y_i\}})= \prod_{i=1}^N k_i(x_i,y_i)
%\]
is universal where $k_i(x,y):= b_i+\delta_i - \max{\{x_i, y_i\}}$.
\end{lem}
For brevity, consider the following proof summary. A complete proof is provided in Appendix A.

\textbf{Sketch of proof:} To show the universality of a kernel, $k$, one must prove that $k$ is continuous and the corresponding RKHS is dense. Now let us consider each $k_i$ in the product $k(x, y)=\prod_{i=1}^n k_i(x_i,y_i)$. As shown in~\cite{JMLR}, each kernel, $k_i$, is continuous, and every triangle function is in the corresponding RKHS, ${\mcl H}_{k_i}$. Also, since $k(x, b_i) = \delta_i$, the constant function is also in ${\mcl H}_{k_i}$. Consequently, we conclude that the RKHS associated with $k_i$ contains a Schauder basis for $\mcl C([a_i, b_i])$ -- implying that the kernel $k_i$ is universal. Since $k$ is the product of $n$ universal kernels, we may now show that $k$ is also universal.% The interested reader can plow through the detailed proof in Appendix~\ref{sec:lemma14proof}.

The following theorem extends Lemma~\ref{lem:universal} and shows that if $k \in \mcl K_T$ is defined by a positive definite parameter, $P \succ 0$, then $k$ is universal.
\begin{thm} \label{thm:universal}
Suppose $k:X\times X \rightarrow \R$ is as defined in Eqn.~\eqref{eqn:tractable} for some $P \succ 0$, $d \in\N$ and $N$ as defined in Eqn.~\eqref{eqn:N}, then $k$ is a universal kernel for $Y = [a - \delta,  b + \delta]$, $X = [a, b]$  and $\delta > 0$
\end{thm}
\begin{proof}
Since $P\succ 0$, then there exists $\epsilon>0$ such that
\[
\hat P = P - \epsilon\bmat{1 & 0 & ... & 0 \\
		      0 & 0 & ... & 0 \\
		      \vdots & \vdots & \ddots & \vdots \\
		      0& 0  & ... & 0 } \geq 0.
\]
Next
\begin{align*}
k(x, y) & = \int_{ Y}   N_T^d(z,x)^T P  N^d_{T}(z,y) dz = \int_{ Y}   N_T^d(z,x)^T \hat P  N^d_{T}(z,y) dz + \epsilon \int_{ Y} \mbf I(z- x) \mbf I(z -y)dz \\
	& = \hat k(x, y) + k_1(x, y)
\end{align*}
where $k_1(x, y) =  \delta  \int_{ Y} \mbf I(z- x) \mbf I(z -y)dz$.

It was shown in Lemma~\ref{lem:universal}, that $k_1(x, y)$ is universal. Since $k$ is a sum of two positive kernels and one of them is universal, then according to \cite{wang2013} and \cite{borgwardt2006} we have that $k:X\times X \rightarrow \R$ is universal for $Y = [a - \delta,  b + \delta]$ and $ X = [a, b]$
\end{proof}

%\begin{thm} \label{thm:universal}
%Suppose $k$ is as defined in Eqn.~\eqref{eqn:tractable} for some $P\geq 0$, $d \in\N$ and $N$ as defined in Eqn.~\eqref{eqn:N}.
%Assume that some leading principal minor of $v^TPv$ is positive definite for $v = \bmat{I \\ -I}$, then $k$ is universal for $Y = [a - \delta,  b + \delta]^N$ and $ X = [a, b]^N$
%\end{thm}
%\begin{proof}
%Firstly, since
%\[
%N^d_{T}(z,x) = \bmat{Z_d(z,x)\mathbf I(z-x) \\ Z_d(z,x)\left(1 - \mathbf I(z-x) \right) } = \bmat{I & 0 \\ -I & I}\bmat{Z_d(z,x)\mathbf I(z-x) \\ Z_d(z,x)}
%\]
%we have
%\begin{align*}
%k(x,y) =& \int_{ Y} \bmat{Z_d(z,x)\mathbf I(z-x) \\ Z_d(z,x)\left(1 - \mathbf I(z-x) \right) }^T P \bmat{Z_d(z,x)\mathbf I(z-x) \\ Z_d(z,x)\left(1 - \mathbf I(z-x) \right) }dz  \\
%=&\int_{\mcl Y}\bmat{Z_d(z,x)\mathbf I(z-x) \\ Z_d(z,x)}^T \bmat{I & -I \\0 & I} P \bmat{I & 0 \\ -I & I}\bmat{Z_d(z,x)\mathbf I(z-x) \\ Z_d(z,x)}dz = \\
%=&\int_{\mcl Y}\bmat{Z_d(z,x)\mathbf I(z-x) \\ Z_d(z,x)}^T \bmat{Q & R \\ R^T & S} \bmat{Z_d(z,x)\mathbf I(z-x) \\ Z_d(z,x)}dz
%\end{align*}
%where $Q = \bmat{I & -I} P \bmat{I \\ -I}$.
%\end{proof}
This theorem implies that even $\mcl K_T^0$ has the universal property.

% \begin{figure}[t]
%
% \begin{center}
% \includegraphics[trim= 10 0 20 20, clip, width=.4\textwidth]{DualityGaps.eps}
%  \end{center}
%         \caption{The Frank-Wolfe error gap and Duality Gap 1000 iterations of Algorithm~\ref{FWTKL}, on two different classification data sets.} \label{fig:gap_convergence}
%  \end{figure}

% \begin{figure}[h]
%
% \begin{center}
% \includegraphics[trim= 10 0 20 20, clip,  width=.4\textwidth]{COMB_CONV_SUM.eps}
%  \end{center}
%         \caption{The normalized objective value error for 1000 iterations of Algorithm~\ref{finalTKL}, on three different classification data sets when compared to the true optimal objective function value calculated using the SDP method in~\cite{JMLR}.} \label{fig:finalconvergence}
% \end{figure}

\begin{figure*}[t]
 \centering
    \begin{subfigure}[t]{0.48\textwidth}
        \centering
\includegraphics[trim= 10 0 35 0, clip, width=1.0\textwidth]{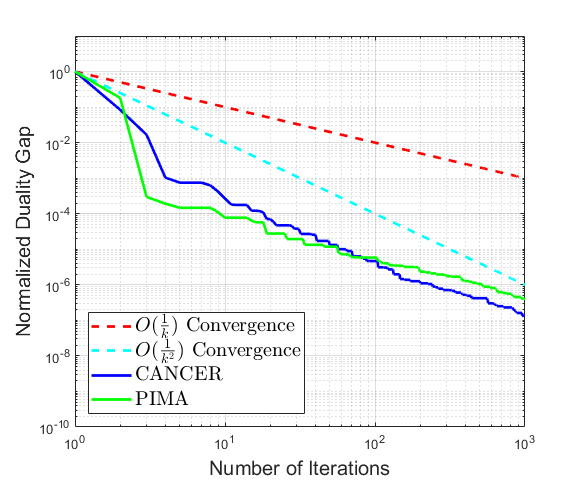}
        \caption{The duality gap for 1000 iterations of FW Algorithm~\ref{FWTKL}, applied to two different classification data sets.}
    \end{subfigure}%
    ~
    \centering
    \begin{subfigure}[t]{0.48\textwidth}
        \centering
\includegraphics[trim= 10 0 35 0, clip, width=1.0\textwidth]{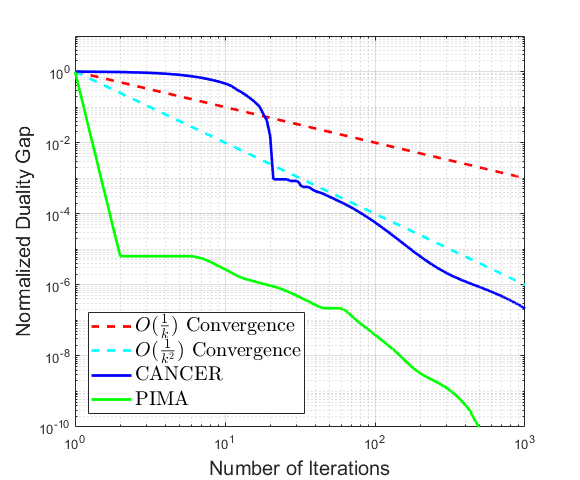}
        \caption{The duality gap for 1000 iterations of APD Algorithm, applied to two different classification data sets.}
    \end{subfigure}%
%    ~
%    \begin{subfigure}[t]{0.3\textwidth}
%        \centering
%        \includegraphics[trim= 10 0 35 0, clip, width=1.0\textwidth]{comb.png}
%        \caption{The duality gap for 1000 iterations of the combined Algorithm~\ref{finalTKL}, applied to two different classification data sets.}
%    \end{subfigure}

    \caption{{\small Convergence rates of the Franke-Wolfe algorithm~\ref{FWTKL} and the alternative APD algorithm described in Subsection~\ref{sec:APD}. In (a) we plot the gap between $OPT\_A(P_k)$ and $OPT\_P(\alpha_k)$ of the Franke-Wolfe Algorithm~\ref{FWTKL} vs. iteration number; in (b) we again plot the gap between $OPT\_A(P_k)$ and $OPT\_P(\alpha_k)$ vs. iteration number for the APD Algorithm and in (c) we plot the boosted algorithm. Both demonstrate { sublinear} convergence, but with enhanced performance for the hybrid algorithm.} }\label{fig:apdconvergence}
\end{figure*}
% \hl{ SUPPLEMENTARY ENDS}

\section{Numerical Convergence and Scalability}\label{sec:scalability}
The computational complexity of the algorithms proposed in this paper will depend both on the computational complexity required to perform each iteration as well as the number of iterations required to achieve a desired level of accuracy. In this section, we use numerical tests to determine the observed convergence rate of Algorithm~\ref{FWTKL} and the observed computational complexity of each iteration when applied to several commonly used machine learning data sets.

\subsection{Convergence Properties}
In this subsection, we briefly consider the estimated number of iterations of the FW algorithm~\ref{FWTKL} required to achieve a given level of accuracy as measured by the gap between the primal and dual solutions. Primal-Dual algorithms such as the proposed FW method typically achieve high rates of convergence. While the number of iterations required to achieve a given level of accuracy does not typically change with the size or type of the problem, if the number of iterations required to achieve convergence is excessive, this will have a significant impact on the performance of the algorithm. In Section~\ref{sec:2step}, we established that the proposed algorithm has worst-case {  $O\left(\frac{1}{k}\right)$ convergence} and proposed an alternative ADP approach with provable { $O\left(\frac{1}{k^2}\right)$ performance}. However, provable bounds on convergence rates are often conservative and in this subsection we examine the observed convergence rates as applied to several test cases in both the classification and regression frameworks.

First, to study the convergence properties of the FW Algorithm~\ref{FWTKL}, in Figure~\ref{fig:apdconvergence}(a), we plot the gap between $OPT\_A(P_k)$ and $OPT\_P(\alpha_k)$ as a function of iteration number for the CANCER and PIMA data sets. The use of the $OPT\_A(P_k)$-$OPT\_P(\alpha_k)$ gap for an error metric is a slight improvement over typical implementations of the FW error metric -- which uses a predicted \textit{bound} on the primal-dual gap. However, in practice, we find that the observed convergence rate does not change significantly depending on which metric is used.
%Also included in Figure~\ref{fig:iteration_convergence}(a) is the duality gap in the original SDP implementation of the TKL algorithm, as obtained from~\cite{TKL_website}.
%We do not include iterations of the SDP primal-dual algorithm as the complexity of these iterations is not comparable to the proposed algorithm.
For reference, Fig.~\ref{fig:apdconvergence} also includes a plot of theoretical worst-case { $O\left(\frac{1}{k}\right)$ and $O\left(\frac{1}{k^2}\right)$} convergence. %In Fig.~\ref{fig:apdconvergence}(b, c), we study the benefits of the ``boosted'' FW-APD algorithm described in Subsection~\ref{sec:APD}.
As is common in primal-dual algorithms, we observe that the achieved convergence rates significantly exceed the provable { sublinear} bound, with this difference being especially noticeable for the first few iterations. These results indicate that for a moderate level of accuracy, the performance of the FW algorithm is adequate -- especially combined with the low per-iteration complexity described in the following subsection. However, benefits of the FW algorithm are more limited at high levels of accuracy. Thus, in Fig. ~\ref{fig:apdconvergence} (b), we find convergence rates for the suggested APD algorithm mentioned in Subsection~\ref{sec:APD}. Unlike the FW algorithm, convergence rates for the first few iterations of APD are not uniformly high. This observation, combined with a slightly higher per-iteration complexity of APD is the reason for our focus on the FW implementation. However, as noted in Appendix B.4, these algorithms can be combined by switching to the APD algorithm after a fixed number of iterations -- an approach which offers superior convergence when desired accuracy is high.

\subsection{Computational Complexity}
\begin{figure*}[t]
    \centering
    \begin{subfigure}[t]{0.45\textwidth}
        \centering
\includegraphics[width =1\textwidth]{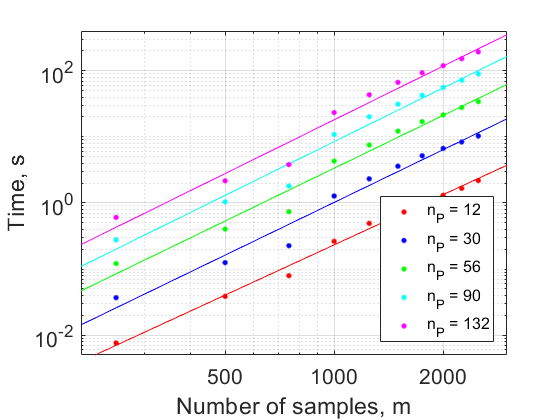}
        \caption{Iteration complexity for the FW algorithm applied to classification}
       % \caption{ FW classification, iteration complexity}
    \end{subfigure}%
    ~
    \begin{subfigure}[t]{0.45\textwidth}
        \centering
\includegraphics[width =1\textwidth]{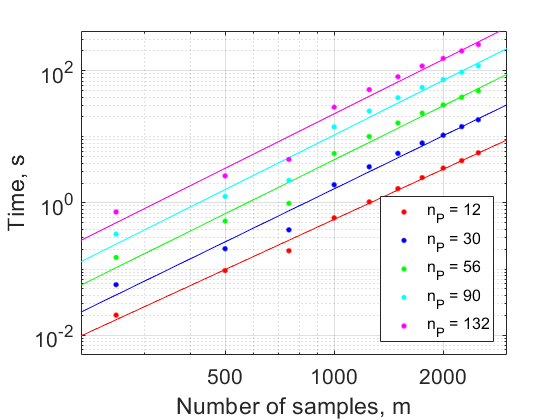}

        \caption{Iteration complexity for the APD algorithm applied to classification}
       %\caption{APD classification, iteration complexity}
    \end{subfigure}
        ~
    \begin{subfigure}[t]{0.45\textwidth}
        \centering
\includegraphics[width =1\textwidth]{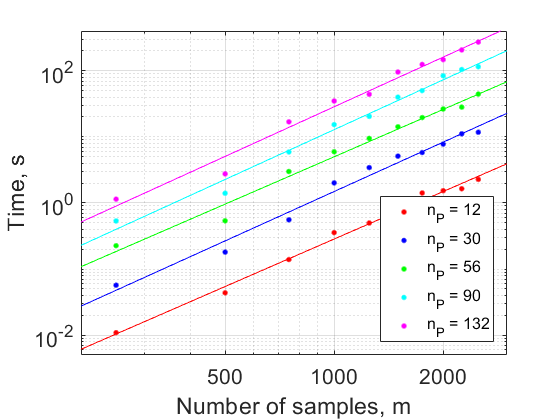}
        \caption{Iteration complexity for the FW algorithm applied to regression}
    \end{subfigure}
    ~
        \begin{subfigure}[t]{0.45\textwidth}
        \centering
\includegraphics[width =1\textwidth]{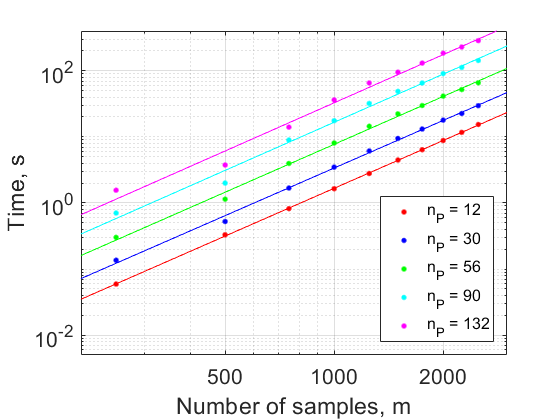}
        \caption{Iteration complexity for the APD algorithm applied to regression}
        %\caption{APD regression, iteration complexity}
    \end{subfigure}

    \caption{{\small Per-iteration complexity of the proposed FW algorithm~\ref{FWTKL} and the alternative APD algorithm described in Subsection~\ref{sec:APD}. In (a) and (c) we find log-log plots of iteration complexity of the Franke Wolfe (FW) TKL classification and regression algorithms, respectively, { as} a function of $m$ for several values of $n_P$. Here $m$ is number of samples and $n_P^{2}$ is the number of parameters in $\mcl K$, so that $P \in \S^{n_P}$.  In { (b) and (d)} we find log-log plots of iteration complexity of the Accelerated Primal Dual (APD) for classification and regression, respectively as a function of $m$ for several values of $n_P$.  In both cases, best linear fit is included for reference. }}\label{fig:kernel_complexity}
\end{figure*}

\begin{table*}[t]
	\centering
	\begin{tabular}{  c |  c c c c } \hline
		Method & Liver   & Cancer  &  Heart  & Pima \\ \hline
		SDP &  95.75 $\pm$ 2.68 &  636.17 $\pm$ 25.43 & 221.67 $\pm$ 29.63 & 1211.66 $\pm$ 27.01 \\
		Algorithm~\ref{FWTKL} &  0.12  $\pm$ 0.03 &  0.41 $\pm$ 0.23 & 4.71 $\pm$ 1.15 & 0.80 $\pm$ 0.36 \\ \hline
	\end{tabular}
	\caption{The mean computation time (in seconds), along with standard deviation, for 30 trials comparing the SDP algorithm in~\cite{JMLR} and Algorithm~\ref{FWTKL}.  All tests are run on an Intel i7-5960X CPU at 3.00 GHz with 128 Gb of RAM.}
	\label{time-table}
\end{table*}

\label{subsec:complexity}
{  
In this subsection, we consider the computational complexity of a single iteration of the proposed FW algorithm~\ref{FWTKL}. Specifically, we examine how the expected complexity of an iteration of each subproblem scales with number of samples and hyperparameters. We then examine how this performance compares with observed complexity in several test cases for both the classification and regression frameworks. Finally, these numerical results are compared with the alternative APD algorithm mentioned in Subsection~\ref{sec:APD}.

The computational complexity of the proposed FW algorithm depends on the size of the data set $m$ (number of samples) and the hyper-parameter $n_P = 2{d+2n+1 \choose d}$, where $d$ is a degree of monomials and $n$ is the number of features. As discussed in Section~\ref{sec:2step},  each iteration of proposed FW algorithm consists of three steps.  
  The first step (\texttt{Step 1a}) of the FW algorithm is the optimization problem $OPT\_ A(P)$ -- an SVC or SVR learning problem.
 For this step, we use a LibSVM implementation for which expected complexity scales as $O(m^{2.3})$. In this step, then kernel defined by matrix $P$ is fixed and hence this step does not depend on $n_P$. We note, however, that observed complexity for this step is a function of the rank of $P$ -- as shown in Appendix~\ref{sec:complexity_SVM}. Specifically, we find that observed complexity of LibSVM is significantly reduced when the matrix $P$ is near optimal in both classification and regression problems.
 
 The second step (\texttt{Step 1b}) entails computing the minimum singular vectors of the $D$ matrix (Eq,.~\eqref{eqn:Dalpha}) of dimension $n_P\times n_P$ -- for which we use the standard Singular Value Decomposition (SVD).
 Before performing an SVD, we must calculate $D$ -- which requires $O(m^2)$ operations for each of the $n_P^2$ elements of the matrix $D$. We note, however, that the cost of these calculations is relatively minor compared with the overall complexity of the SVD and SVC/SVR from \texttt{Step 1a}. For solving the SVD, we use a LAPACK implementation which has worst-case complexity which scales as $O(n_P^3)$ -- and which does not depend on $m$ (the number of samples). The last step, \texttt{Step 2} is a primitive line search algorithm, where for each iteration, we evaluate a fixed number of $n_\gamma$ candidate step sizes ($\gamma$). Each candidate step size requires solving a QP problem ($O(m^{2.3})$), leading to worst-case complexity scaling as $O(m^{2.3})$. We conclude that the expected complexity of the proposed algorithm scales as $O(m^{2.3} + m^2n_P^2 + n_P^3)$.

To compare expected iteration complexity with observed complexity, we next test the proposed FW algorithm on several test cases and compare these results with observed complexity of the alternative APD algorithm mentioned in Subsection~\ref{sec:APD}.}
 In Figures~\ref{fig:kernel_complexity}(a-d), we find the computation time of a single iteration of the FW TKL and APD algorithms for both classification and regression on an Intel i7-5960X CPU with 128 Gb of RAM as a function of $m$ for several values of $n_P$, where $m$ is the number of samples used to learn the TK kernel function and the size of $P$ is $n_P \times n_P$ (so that $n_P$ is a function of the number of features and the degree of the monomial basis $Z_d$). The data set used for these plots is California Housing (CA) in~\cite{pace1997sparse}, containing 9 features and $m = 20,640$ samples.   In the case of classification, labels with value greater than or equal to the median of the output were relabeled as $1$, and those less than the median were relabeled as $-1$. Figures~\ref{fig:kernel_complexity}(a-d) demonstrate that the complexity of Algorithm~\ref{FWTKL} { (Algorithm 4)} scales as approximately $O( m^{2.6}n_P^{1.8})$ ($O(m^{2.7}n_P^{1.6})$ ) for classification and $O(m^{2.4}n_P^{1.9})$ ($O( m^{2.4}n_P^{1.2})$) for regression. While this complexity scaling is consistent with theoretical bounds, and while the difference in iteration complexity between the FW and APD algorithms for these data sets is minimal, we find significant differences in scaling between data sets. Furthermore, we find that for similar scaling factors, the FW iteration is approximately 3 times faster than the APD iteration. This multiplicative factor increases to 100 when compared to the SDP algorithm in~\cite{JMLR}. This factor is illustrated for classification using four data sets in Table~\ref{time-table}.

%\begin{table*}[b]
%	\centering
%	\begin{tabular}{  c |  c c c  } \hline
%		 & SDP   & FW &  APD \\ \hline
%		CCPP &  $1.7\times10^{-8 }m^{2.4} n_p^{1.9}$ &
%                       $1.8\times10^{-7}m^{1.6} n_p^{0.7}$ &
%			       $1.1\times10^{-7}m^{2.3} n_p^{0.3}$ \\
%		CA &  $3.4\times10^{-8 }m^{2.2} n_p^{1.9}$ &
%                    $1.1\times10^{-10}m^{2.4} n_p^{1.9}$ &
%			    $3.6\times10^{-9 }m^{2.4} n_p^{1.3}$ \\\hline
%	\end{tabular}
%	\caption{Observed time (in seconds) of a single iteration for the SDP algorithm in~\cite{JMLR}, Algorithm~\ref{FWTKL} and APD algorithm for two data sets as a function of $m$ and $n_p$.  All tests are run on an Intel i7-5960X CPU at 3.00 GHz with 128 Gb of RAM.}
%	\label{tab:complexity}
%\end{table*}
%\section{Methodology of Numerical Analysis} \label{sec:Data}

%
%References to the original sources for the data sets used in Sections~\ref{sec:8} of the paper are included in Table~\ref{table:Data}. These data sets were all obtained from the UCI Machine Learning Repository~ \citep[see][]{UCI}) or from LIBSVM~ \citep[see][]{LibSVM}). For convenience, these data sets are included in our software. For the regression algorithm as applied to the Abalone data set (which can be used for classification or regression), we labelled the abalone as mature ($y_i=1$) if the number of rings are greater than or equal to 9 and immature $y_i=-1$ otherwise.
%
%To compare the proposed algorithm to other machine learning algorithms we use the following implementation.
%

\section{Accuracy of the New TK Kernel Learning Algorithm}\label{sec:8}

In this section, we compare the accuracy of the classification and regression solutions obtained from the FW TKL algorithm with $N$ as defined in Eq.~\eqref{eqn:N} to the accuracy of SimpleMKL, Neural Networks, Random Forest, and XGBoost algorithms.  %The Neural Network, Random Forest and XGBoost algorithms are not convex, and the SimpleMKL algorithm uses classes of kernel functions which do not satisfy all three properties of tractability, universality and density.
In the case of classification we also include three algorithms from the MKLpy python toolbox (AMKL, PWMK, and CKA).% that similarly do not optimize classes of kernels that are tractable, universal, and dense.

%References to the original sources for the data sets used in Sections~\ref{sec:8} of the paper are included in Table~\ref{table:Data}. These data sets were all obtained from the UCI Machine Learning Repository~ \citep[see][]{UCI}) or from LIBSVM~ \citep[see][]{LibSVM}). For convenience, these data sets are included in our software. For the regression algorithm as applied to the Abalone data set (which can be used for classification or regression), we labelled the abalone as mature ($y_i=1$) if the number of rings are greater than or equal to 9 and immature $y_i=-1$ otherwise.

References to the original sources for the data sets used in Section~\ref{sec:8} of the paper are included in Table~\ref{table:Data}. Six classification and six regression data sets were chosen { arbitrarily} from~\cite{UCI} and~\cite{LibSVM} to contain a variety of number of features and number of samples. %No other data sets were tested for relative performance and data sets were not ``pre-screened''.
In both classification and regression, the accuracy metric uses 5 random divisions of the data into test sets ($m_t$ samples $\cong 20$\% of data) and training sets ($m$ samples $\cong 80$\% of data). For regression, the training data is used to learn the kernel and predictor. The predictor is then used to predict the test set outputs.
\begin{table*}[t]
  \centering
\begin{tabular}{  c |  c c  c } \hline
Name   & Type & Source & References  \\ \hline
Liver  & Classification & UCI  & \cite{mcdermott2016diagnosing}  \\ \hline
Cancer & Classification & UCI & \cite{mangasarian1990pattern}  \\ \hline
Heart  & Classification & UCI & No Associated Publication  \\ \hline
Pima   & Classification & UCI & No Associated Publication  \\ \hline
Hill Valley & Classification  & UCI & No Associated Publication \\ \hline
Shill Bid   & Classification  & UCI & \cite{alzahrani2018scraping,alzahrani2020clustering}  \\ \hline
Abalone     & Classification  & UCI & \cite{waugh1995extending}  \\ \hline
Transfusion     & Classification &  UCI & \cite{yeh2009knowledge}  \\ \hline
German     & Classification  & LIBSVM  & No Associated Publication \\ \hline
Four Class     & Classification  &  LIBSVM & \cite{ho1996building} \\ \hline
Gas Turbine & Regression  & UCI  & \cite{kaya2019predicting}  \\ \hline
Airfoil     & Regression  & UCI & \cite{brooks1989airfoil}  \\ \hline
CCPP        & Regression  & UCI & \cite{tufekci2014prediction,kaya2012local}  \\ \hline
CA                 & Regression  & LIBSVM & \cite{pace1997sparse} \\ \hline
Space              & Regression  & LIBSVM & \cite{pace1997quick} \\ \hline
Boston Housing     & Regression  & LIBSVM & \cite{harrison1978hedonic} \\ \hline
\end{tabular}
  \caption{References for the data sets used in Section~\ref{sec:8}.  All data sets are available on the UCI Machine Learning Repository or from the LIBSVM database.}
  \label{table:Data}
\end{table*}
\paragraph{Regression analysis} Using six different regression data sets, the MSE accuracy of the proposed algorithm (TKL) with $N$ as defined in Eq.~\eqref{eqn:N} was below average on five of the data sets, an improvement over all other algorithms but XGBoost which also scored above average on five of the data sets.
To evaluate expected improvement in accuracy,
we next compute the average MSE improvement for TKL averaged over all algorithms and data sets to be $23.6\%$ -- i.e.

%\[
%\sum_{j=1}^6 \frac{MSE_{TKL,j}}{5}-\frac{\sum_{i=1}^5 \sum_{j=1}^6 MSE_{i,j}}{30}=-23.6\%
%\]
\[
\frac{1}{6}\sum_{j=1}^6 \frac{MSE_{TKL,Data set(j)}}{\frac{1}{5} \sum_{i=1}^5 MSE_{Algorithm(i), Data set(j)}}\cdot 100=100\% - 23.6\%
\]
This improvement in average performance was better than all other tested algorithms including XGBoost.  %We show the percent decrease in the MSE of the top three performing methods per data set when compared to the average MSE.  TKL with the TK kernel is the top performing algorithm for three of these data sets.

Predictably, the computational time of TKL is significantly higher than non-convex non-kernel-based approaches such as RF or XGBoost. However, the computation time of TKL is lower than other kernel-learning methods such as SMKL --- note that for large data sets ($m > 5000$) TKL is at least 20 times faster than SMKL. Surprisingly, the computation time of TKL is comparable to over-parameterized non-convex stochastic descent methods such as NNet.
\begin{table*}[t]
  \centering {\scriptsize
\begin{tabular}{  c |  c  c c || c |  c  c c }\hline
Data set & Method  &   Error   & Time (s)  & Data set & Method  &   Error & Time (s) \\ \hline
  Gas  & TKL & 0.23 $\pm$ 0.01 &\hspace{-1mm} 13580 $\pm$ 2060  & CCPP & TKL & 10.57 $\pm$ 0.82 &\hspace{-2mm} 626.7 $\pm$ 456.0 \\
Turbine & SMKL & N/A & N/A &  $n$ = 4 & SMKL & 13.93 $\pm$ 0.78 & \hspace{-3mm}13732 $\pm$ 1490 \\
 $n$ = 11 &  NNet & 0.27 $\pm$ 0.03  & 1172 $\pm$ 100 & $m$ = 8000 & NNet & 15.20 $\pm$ 1.00 &\hspace{-1mm} \hspace{-2mm}305.71 $\pm$ 9.25 \\
\hspace{-1mm}$m$ = 30000 &  RF & 0.38 $\pm$ 0.02 & 16.44 $\pm$ 0.57 & $m_t$ = 1568 &  RF & 10.75 $\pm$ 0.70 & 1.65 $\pm$ 0.19 \\
 \hspace{-1mm}$m_t$ = 6733&\hspace{-1mm} XGBoost  &\hspace{-2mm} 0.33 $\pm$ 0.005  & 49.46 $\pm$ 1.93  &
 & \hspace{-1mm}XGBoost  & 8.98 $\pm$ 0.81 & 5.47 $\pm$ 2.73  \\ \hline
Airfoil & TKL & 1.41 $\pm$ 0.44 & 49.87 $\pm$ 4.29  & CA &  TKL & .012 $\pm$ .0003 & 1502 $\pm$ 2154 \\
$n$ = 5 & SMKL & 4.33 $\pm$ 0.79 & \hspace{-1mm}617.8 $\pm$ \hspace{-1mm}161.6 & $n=8$ & SMKL & N/A & N/A \\
$m$ = 1300 & NNet & 6.06 $\pm$ 3.84 & 211.9 $\pm$ 41.0 & $m=16500$ & NNet & \hspace{-1mm}.0113 $\pm$ .0004 &\hspace{-1mm}914.3 $\pm$ 95.9 \\
$m_t$ = 203 &  RF & 2.36 $\pm$ 0.42 & 0.91 $\pm$ 0.20 &
$m_t=4140$ & RF & \hspace{-1mm}.0096 $\pm$ .0003 & 5.28 $\pm$ 3.13 \\
 &\hspace{-1mm} XGBoost  & 1.51 $\pm$ 0.40  & 2.59 $\pm$ 0.06 &
 & \hspace{-1mm}XGBoost &\hspace{-2mm} .0092 $\pm$ .0002 & 5.28 $\pm$ 3.13   \\ \hline
Space & TKL & .013 $\pm$ .001  &  \hspace{-1mm}121.8 $\pm$ 49.2   & Boston  & TKL & \hspace{-1mm}10.36 $\pm$ 5.80 & \hspace{-1mm}63.05 $\pm$ 2.90 \\
$n$ = 12 & SMKL & .019 $\pm$ .005  & 3384 $\pm$ 589  &
Housing &  SMKL & \hspace{-1mm}15.46 $\pm$ 11.49 & \hspace{-1mm}10.39 $\pm$ 0.89 \\
$m$ = 6550 & NNet & .014 $\pm$ .004 & \hspace{-1mm}209.7 $\pm$ 37.4 &
$n$ = 13 & NNet  &\hspace{-2mm} 50.90 $\pm$ 44.19 & 79.2 $\pm$ 42.8 \\
$m_t$ = 1642  &  RF & .017 $\pm$ .003  & 1.06 $\pm$ 0.27 & $m$ = 404 & RF  & 10.27 $\pm$ 5.70 & 0.68 $\pm$ 0.40  \\
 &\hspace{-1mm} XGBoost  & .015 $\pm$ .002  & 0.32 $\pm$ 0.02  &
$m_t$ = 102  & \hspace{-1mm} XGBoost  & 9.40 $\pm$ 4.17 & 0.14 $\pm$ 0.06  \\ \hline
\end{tabular} }
  \caption{{\small Regression performance of Tessellated Kernel learning for 6 regression data sets with comparison of 5 different ML algorithms. Each measurement was repeated 5 times. The resulting test Mean Squared Error (MSE) and training time are included in the table.  All tests are run on a desktop with Intel i7-5960X CPU at 3.00 GHz and with 128 Gb of RAM.  N/A indicates that the algorithm was stopped after 24 hours without a solution.  For each data set. the number of samples, $m$, the size of the test part $m_t$ and the number of features, $n$, are presented.}}
  \label{table:Reg}
\end{table*}

\paragraph{Classification analysis} Using six classification data sets and comparing 7 algorithms, the TSA of the proposed TKL algorithm was above average on all of the data sets, an improvement over all other algorithms. Next, we compute the average improvement in accuracy of TKL over average TSA for all algorithms to be 6.77$\%$ -- i.e.

\[
\frac{1}{6}\sum_{j=1}^6 \frac{TSA_{TKL,Data set(j)}}{\frac{1}{8}\sum_{i=1}^8 TSA_{Algorithm(i), Data set(j)}}\cdot 100=100\% + 6.77\%
\]
This was close to the top score of 6.84$\%$ achieved by the AMKL algorithm (The PWMK algorithm failed to converge on one data set, and the TSA from this test was not included in the calculation).  %Table~\ref{table:Class} shows the percent increase in the TSA of the top three performing methods when compared to the average TSA per data set.  TKL with the TK kernel is the top performing algorithm for two of these data sets.

Again, the computational time of TKL is significantly higher than RF or XGBoost, but comparable to other kernel learning methods and NNet. Unlike TKL, the computational time of other MKL methods is highly variable and often does not seem to scale predictably with the number of samples and features -- e.g. PWMK for FourClass and Shill Bid data sets and AMKL for Transfusion and German, where the computational time is much higher for smaller data sets.

\begin{table*}[t]
  \centering {\scriptsize
\begin{tabular}{  c |  c  c c || c |  c  c c }\hline
Data set & Method  &   Accuracy (\%)   & Time (s)  & Data set & Method        &   Accuracy (\%)               & Time (s) \\ \hline
Abalone & TKL & 84.61 $\pm$ 1.60 & 17.63 $\pm$ 3.77  & Hill Valley &  TKL  & 86.70 $\pm$ 5.49 & 86.7 $\pm$ 48.2 \\
$n$ = 8  & SMKL & 83.13 $\pm$ 1.06  &  \hspace{-2mm}350.4 $\pm$ 175.1 & $n$ = 100  &  SMKL & 51.23 $\pm$ 3.55 & 2.81 $\pm$ 2.83  \\
$m$ = 4000 & NNet & 84.70 $\pm$ 1.82 & 4.68 $\pm$ 0.64 & $m$ = 1000 & NNet  & 70.00 $\pm$ 4.79 & 3.79 $\pm$ 1.75 \\
$m_t$ = 677 & RF  & 84.11 $\pm$ 1.33 & 0.98 $\pm$ 0.21 & $m_t$ = 212 & RF  & 56.04 $\pm$ 3.27 & 0.75 $\pm$ 0.33 \\
 & \hspace{-1mm}XGBoost  & 82.69 $\pm$ 1.06  & 0.20 $\pm$ 0.06 &  & \hspace{-1mm}XGBoost  & 55.66 $\pm$ 2.37 & 0.58 $\pm$ 0.34 \\
  &  AMKL & 84.64 $\pm$ 1.01   &  0.95 $\pm$ 0.07  &  &  AMKL & 94.71 $\pm$ 1.72 & 5.50 $\pm$ 3.84  \\
  & PWMK & 84.64 $\pm$ 1.01  & 3.13 $\pm$ 0.12  &  & PWMK  & 94.34 $\pm$ 1.69  & \hspace{-2mm}13.10 $\pm$ 5.19  \\
  & CKA  & 65.05 $\pm$ 0.76 &\hspace{-2mm} 21.43 $\pm$ 0.32  & & CKA  & 47.92 $\pm$ 0.57 & 0.50 $\pm$ 0.08 \\  \hline
Transfusion & TKL & 77.84 $\pm$ 3.89 & 0.25 $\pm$ 0.08  & Shill Bid & TKL & 99.76 $\pm$ 0.08 & \hspace{-2mm}23.66 $\pm$ 2.63\\
$n$ = 4  & SMKL & 76.62 $\pm$ 4.79 & 2.44 $\pm$ 3.08 & $n$ = 9  &  SMKL & 97.71 $\pm$ 0.32 & 81.0 $\pm$ 13.1 \\
$m$ = 600 & NNet & 78.78 $\pm$ 3.26 & 1.01 $\pm$ 0.47 & $m$ = 5000 & NNet  & 98.64 $\pm$ 0.86 & 3.56 $\pm$ .60 \\
$m_t$ = 148  &  RF & 75.00 $\pm$ 3.58 & 0.54 $\pm$ 0.24  & $m_t$ = 1321 & RF  & 99.35 $\pm$ 0.14 & 0.78 $\pm$ 0.36 \\
 & \hspace{-1mm}XGBoost  & 73.92 $\pm$ 3.95 & 0.13 $\pm$ 0.11  &  &\hspace{-1mm} XGBoost  & 99.61 $\pm$ 0.06 & 0.13 $\pm$ 0.04  \\
  & AMKL  & 74.46 $\pm$ 1.50  & \hspace{-2mm}766.9 $\pm$ 315.4  &  & AMKL  & 99.72 $\pm$ 0.10 & 1.24 $\pm$ 0.04  \\
   & PWMK  & N/A &  N/A &  & PWMK  & 99.72 $\pm$ 0.10 & 3.03 $\pm$ 0.04  \\
    & CKA  & 76.35 $\pm$ 4.27  & 0.18 $\pm$ 0.03  &  & CKA  & 99.65 $\pm$ 0.21 & 55.8 $\pm$ 0.9  \\ \hline
German & TKL & 75.80 $\pm$ 1.89 & 58.7 $\pm$ 36.1 & FourClass & TKL & 99.77 $\pm$ 0.32 & 0.13 $\pm$ 0.01 \\
$n$ = 24  & SMKL & 74.30 $\pm$ 3.55 & 17.78 $\pm$ 4.79 & $n$ = 2  &  SMKL & 94.53 $\pm$ 12.2  & 0.85 $\pm$ 0.48 \\
$m$ = 800 & NNet & 72.70 $\pm$ 3.98 & 0.61 $\pm$ 0.05 & $m$ = 690 & NNet  & 99.99 $\pm$ 0.01 & 0.53 $\pm$ 0.03  \\
$m_t$ = 200  &  RF & 74.90 $\pm$ 1.35 & 0.64 $\pm$ 0.28  & $m_t$ = 172 & RF  & 99.30 $\pm$ 0.44 & 0.68 $\pm$ 0.53  \\
 & \hspace{-1mm}XGBoost  & 72.40 $\pm$ 2.89 & 0.10 $\pm$ 0.03 &  & \hspace{-1mm}XGBoost  & 98.95 $\pm$ 0.44 & 0.04 $\pm$ 0.00  \\
  & AMKL  & 70.80 $\pm$ 1.47 & 2.88 $\pm$ 0.38 &  & AMKL  & 99.99 $\pm$ 0.01 & 1.39 $\pm$ 0.03   \\
   & PWMK  & 70.70 $\pm$ 1.50 & \hspace{-2mm}907.0 $\pm$ 77.6 &  & PWMK  & 99.99 $\pm$ 0.01 & 990 $\pm$ 60.2  \\
    & CKA  & 68.50 $\pm$ 1.58 & 0.25 $\pm$ 0.03  &  & CKA  & 66.16 $\pm$ 3.44  & 0.16 $\pm$ 0.03 \\ \hline
\end{tabular} }
  \caption{{\small Classification performance of Tessellated Kernel learning for  6 different classification data  sets with comparison of 7 ML methods. Each measurement was repeated 5 times. The resulting Test Set Accuracy and training time are presented in the table.  All tests are run on a desktop with Intel i7-5960X CPU at 3.00 GHz and with 128 Gb of RAM.  N/A indicates that the algorithm was stopped after 24 hours without a solution. For each data set. the number of samples, $m$, the size of the test part $m_t$ and the number of features, $n$, are presented.}}
  \label{table:Class}
\end{table*}
Details of the implementation of the algorithms used in this study are as follows.

\noindent\textbf{[TKL]} Algorithm~2 with $N$ as defined in Eqn.~\eqref{eqn:N}, where $Z_d$ is a vector of monomials of degree $d=1$ or less. The regression problem is posed using $\epsilon=.1$. The data is scaled so that $x_i \in [0,1]^n$ and $[a,b] = [0-\delta,1+\delta]^n$, where $\delta \geq 0$ and $C$ in the kernel learning problem are chosen by 2-fold cross-validation. Implementation and documentation of this method is described in Appendix D.1 and is publicly available via Github~\citep{githubTKL}; 

\noindent\textbf{[SMKL]} SimpleMKL proposed in~\cite{rakotomamonjy_2008} with a standard selection of Gaussian and polynomial kernels with bandwidths arbitrarily chosen between .5 and 10 and polynomial degrees one through three - yielding approximately $13(n+1)$ kernels. The regression and classification problems are posed using $\epsilon = .1$ and $C$ is chosen by 2-fold cross-validation;

\noindent\textbf{[NNet]} A neural network with 3 hidden layers of size 50 using MATLAB's \texttt{patternnet} for classification and \texttt{feedforwardnet} for regression where learning is halted after the error in a validation set decreased sequentially 50 times;

\noindent\textbf{[RF]} The Random Forest algorithm as in~\cite{Breiman2004RandomF} as implemented on the scikit-learn python toolbox~\citep[see][]{scikit-learn}) for classification and regression.  Between 50 and 650 trees (in 50 tree intervals) are selected using 2-fold cross-validation;

\noindent\textbf{[XGBoost]} The XGBoost algorithm as implemented in~\cite{chen2016xgboost} for classification and regresion.  Between 50 and 650 trees (in 50 tree intervals) are selected using 2-fold cross-validation;

\noindent\textbf{[AMKL]} The AverageMKL implementation from the MKLpy python package proposed in~\cite{lauriola2020mklpy} -- averages a standard selection of Gaussian and polynomial kernels;

\noindent\textbf{[PWMK]} The PWMK implementation from the MKLpy python package proposed in~\cite{lauriola2020mklpy}, which uses a heuristic based on individual kernel performance as in~\cite{tanabe2008simple} to learn the weights of a standard selection of Gaussian and polynomial kernels;

\noindent\textbf{[CKA]} The CKA implementation from the MKLpy python package~\citep{lauriola2020mklpy}, uses the centered kernel alignment optimization in closed form  as in~\cite{cortes2010two} to learn the weights of a standard selection of Gaussian and polynomial kernels.

% \noindent\textbf{[TKL]} Algorithm~\ref{FWTKL} with $d=1$, $\epsilon = .1$ and we scale the data so that $x_i \in [0,1]^n$, and then select $[a,b] = [0-\delta,1+\delta]^n$, where $\delta\geq0$ and $C$ are chosen by 2-fold cross-validation;

% \noindent\textbf{[SMKL]} SimpleMKL \cite{rakotomamonjy_2008} with a standard selection of Gaussian and polynomial kernels with bandwidths arbitrarily chosen between .5 and 10 and polynomial degrees one through three - yielding approximately $13(n+1)$ kernels. We set $\epsilon = .1$ as in TKL and $C$ is chosen by 2-fold cross-validation;

% \noindent\textbf{[NNet]} A neural network with 3 hidden layers of size 50 using MATLABs (\texttt{patternnet} for classification and \texttt{feedforwardnet} for regression) implementation and stopped learning after the error in a validation set decreased sequentially 50 times.

% \noindent\textbf{[RF]} The Random Forest algorithm~\cite{Breiman2004RandomF} as implemented on the scikit-learn python toolbox~\cite{scikit-learn} for classification and regression.  We select between 50 and 650 trees (in 50 tree intervals) using 2-fold cross-validation.

To further illustrate the importance of density property and the TKL framework for practical regression problems, Algorithm 2 with $N=N_2^T$ was applied to elevation data from~\cite{becker2009global} to learn a SVM predictor representing the surface of the Grand Canyon in Arizona. This data set is particularly challenging due to the variety of geographical features. The results can be seen in Figure~\ref{fig:GC}(d) where we see that the regression surface visually resembles a photograph of this terrain, avoiding the artifacts present in the SVM from an optimized Gaussian kernel seen in Figure~\ref{fig:GC}(c).

\begin{figure*}[t]
    \centering
    \begin{subfigure}[t]{0.23\textwidth}
        \centering
\includegraphics[trim= 20 0 50 20, clip, width=.95\textwidth]{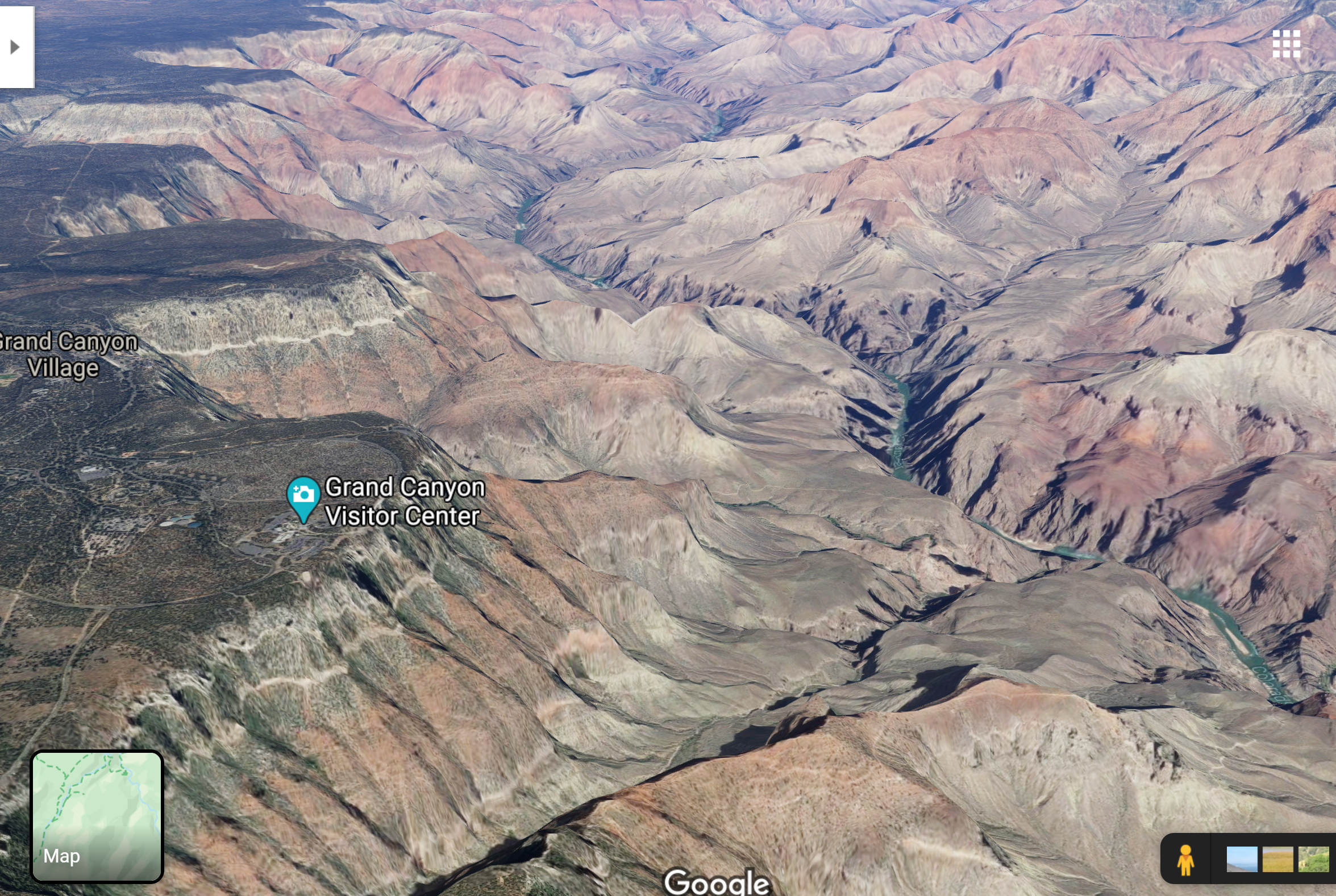}
        \caption{An image from Google Maps of a section of the Grand Canyon corresponding to (36.04, -112.05) latitude and (36.25, -112.3) longitude.}
    \end{subfigure}%
    ~
    \begin{subfigure}[t]{0.23\textwidth}
        \centering
\includegraphics[trim= 20 0 50 20, clip, width=0.95\textwidth]{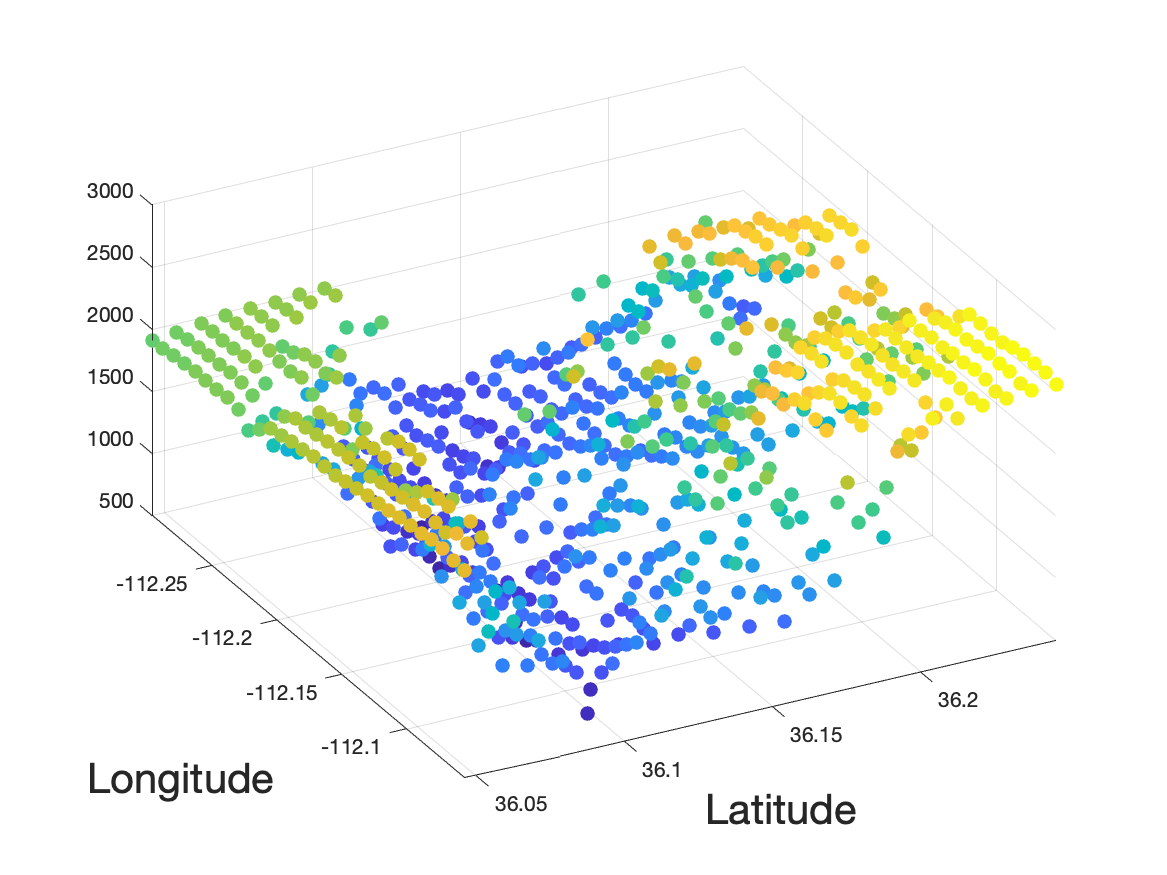}
        \caption{Elevation data ($m=750$) from \cite{becker2009global} for a section of the Grand Canyon between (36.04, -112.05) latitude and (36.25, -112.3) longitude.}
    \end{subfigure}
        ~
    \begin{subfigure}[t]{0.23\textwidth}
        \centering
\includegraphics[trim= 20 0 50 20, clip, width=0.95\textwidth]{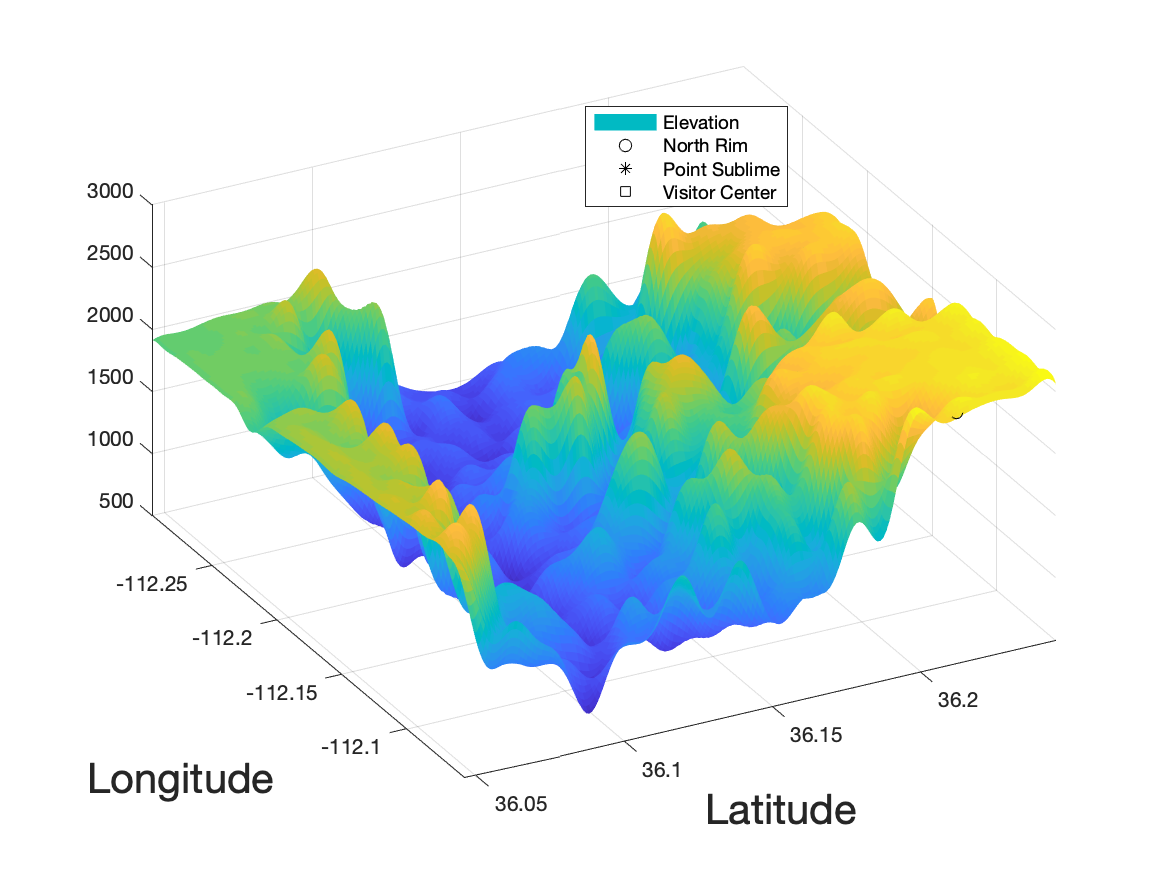}
        \caption{Predictor using a hand-tuned Gaussian kernel trained on the elevation data in (b). The Gaussian predictor poorly represents the sharp edge at the north and south rim.}
    \end{subfigure}
    ~
        \begin{subfigure}[t]{0.23\textwidth}
        \centering
\includegraphics[trim= 20 0 50 20, clip, width=0.95\textwidth]{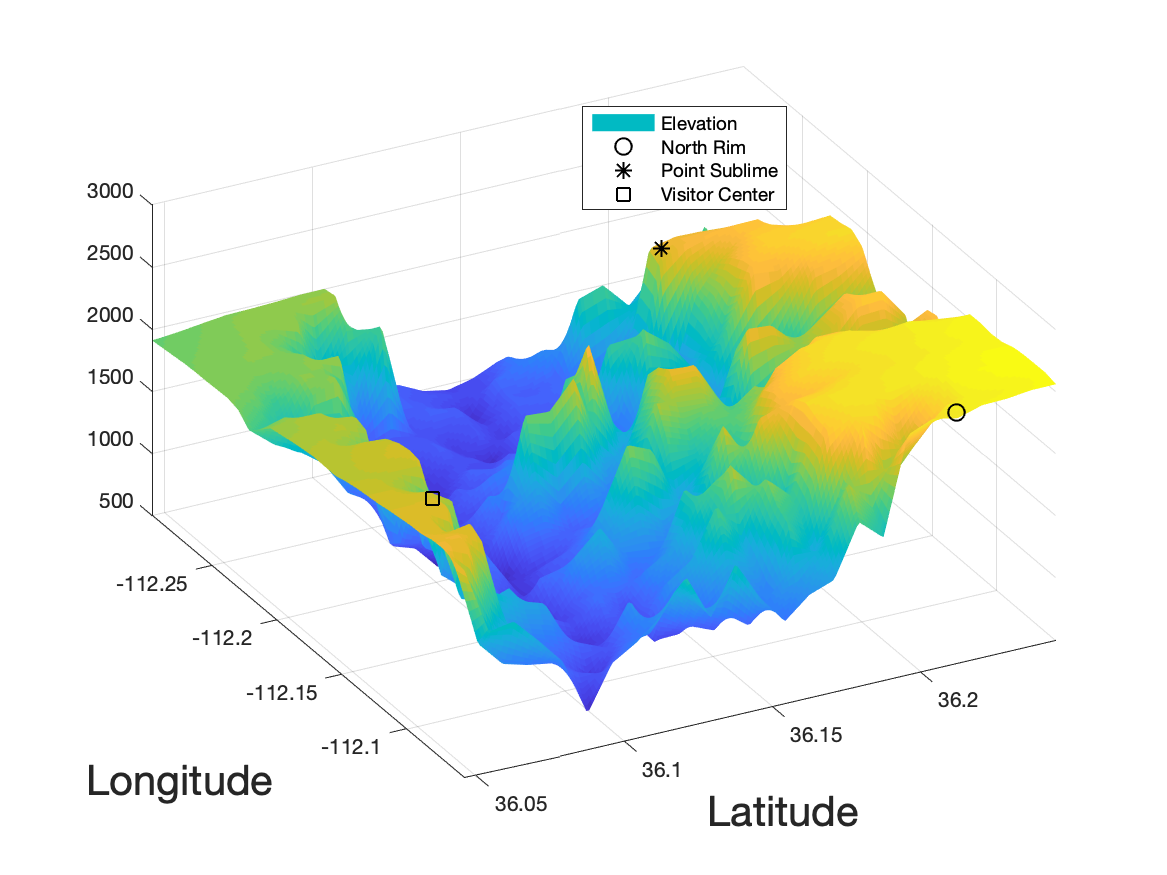}
        \caption{Predictor from Algorithm~\ref{FWTKL} trained on the elevation data in (b). The TK predictor accurately represents the north and south rims of the canyon.}
    \end{subfigure}
    \caption{{\small Subfigure (a) shows an 3D representation of the section of the Grand Canyon to be fitted. In (b) we plot elevation data of this section of the Grand Canyon. In (c) we plot the predictor for a hand-tuned Gaussian kernel. In (d) we plot the predictor from Algorithm~\ref{FWTKL} for $d=2$. }}\label{fig:GC}
\end{figure*}

\section{Conclusion}
We have proposed a generalized kernel learning framework -- using positive matrices to parameterize positive kernels. While such problems can be solved using semidefinite programming, the use of SDP results in large computational overhead. To reduce this computational complexity, we have proposed a saddle-point formulation of the generalized kernel learning problem. This formulation leads to a primal-dual decomposition, which can then be solved efficiently using algorithms of the Frank-Wolfe or accelerated primal dual type -- with corresponding theoretical guarantees of convergence. In both cases, we have shown that the primal and dual sub-problems can be solved using a singular value decomposition and quadratic programming, respectively. Numerical experiments confirm that the FW-based algorithm is approximately 100 times faster than the previous SDP algorithm from~\cite{JMLR}. %Although If the high precision is required, we recommend to use the APD algorithm.
Finally, 12 large and randomly selected data sets were used to test accuracy of the proposed algorithms compared to 8 existing state-of-the-art alternatives -- yielding uniform increases in accuracy with similar or reduced computational complexity.

\acks{We would like to acknowledge support for this project
from the National Science Foundation grant CCF 2323532 and National Institutes of Health grant NIH R01 GM144966.}


\newpage
\appendix

\section{Proof of Universality of Tessellated Kernels}\label{sec:lemma14proof}
In this appendix, we find a detailed proof of Lemma~\ref{lem:universal}.
\setcounter{thm}{13}
\begin{lem}
For any $a,b,\delta \in \R^n$ with $a<b$ and $\delta>0$, let $Y = [a - \delta,  b + \delta]$ and $ X = [a, b]$. Then the kernel
\begin{align*}
k(x, y) &= \int_{  Y} \mbf I(z- x) \mbf I(z -y)dz  = \int_{\substack{x \leq z \\ y \leq z  \\ z \leq b+\delta\\}} 1 dz   = \prod_{i=1}^n (b_i+\delta_i - \max{\{x_i, y_i\}})= \prod_{i=1}^n k_i(x_i,y_i)
\end{align*}
%\[
%k(x, y) = \int_{  Y} \mbf I(z- x) \mbf I(z -y)dz = \prod_{i=1}^N (b_i+\delta_i - \max{\{x_i, y_i\}})= \prod_{i=1}^N k_i(x_i,y_i)
%\]
is universal where $k_i(x,y):= b_i+\delta_i - \max{\{x_i, y_i\}}$.
\end{lem}
\begin{proof} As per~\cite{micchelli_2006}, $k:X \times X \rightarrow \R$ is universal if it is continuous and, if for any $g\in \mcl C(X)$ and any $\epsilon>0$ there exist $f \in \mcl H $ such that $\norm{f-g}_{\mcl C}\le \epsilon$ where
\[
\mcl H:=\{f \;:\; f(x)=\sum_{i = 1}^{m}\alpha_i k(x, y_i)\;:\; y_i \in X,\; \alpha_i\in \R,\; m\in \N \}.
\]

First, we show that each kernel, $k_i(x,y)$, with corresponding RKHS $\mcl H_i$, is universal for $i=1,\cdots,n$. As shown in~\cite{JMLR}, any triangle function, $\Lambda$,  of height $1$, of width $2 \beta$, and centered at $y_2\in [a_i+\beta,b_i-\beta]$,  is in $\mcl H_i$, since
\[
\Lambda(x) = \sum_{j =1}^3 \alpha_j k_i(x, y_j) = \begin{cases}
0, & \text{ if } x < y_1 \\
\frac{1}{\beta}(x - y_1), & \text{ if } y_1 \leq x < y_2 \\
1-\frac{1}{\beta}(x - y_2), & \text{ if } y_2 \leq x < y_3 \\
0, & \text{ if }  y_3 \leq x \\
\end{cases}
\]
where $y_1=y_2-\beta$, $y_3= y_2+\beta$ and $\alpha_1 = \alpha_3 = -\frac{1}{\beta} $, $\alpha_2 = 2\frac{1}{\beta}$. Furthermore, using $\alpha_1=\frac{1}{\delta}$ and $y_1=b_i$, we have $\alpha_1k_i(x,y_1)=\frac{1}{\delta}k_i(x, b_i) =  1$  and so the constant function $\Lambda=1$ is also in $\mcl H_i$. Thus we conclude that $\mcl H_i$ contains the Schauder basis for $\mcl C([a_i, b_i])$~\citep[See][]{hunter2001applied} and hence the kernel is universal for $n = 1$.

Next, since $k=\prod_{i=1}^n k_i$, we have that
\[
\mcl H=\left\{\prod_{i=1}^n f_i\;:\; f_i \in \mcl H_i\right\}.
\]
Thus, since the $k_i$ are universal, for any for any $\gamma \in \N^n$ and $x_i^{\gamma_i} $, there exist $g_i \in \mcl H_i$ such that $|x_i^{\gamma_i}-g_i(x_i)|\le \epsilon$.

Note that $g_i(x_i)$ are bounded, indeed define $C_i:=   (\max\{|a_i|, |b_i|\})^{\gamma_i}  + 1$ and $C := \prod_{i=1}^n C_i$, then
\[
\sup_{x_i \in [a_i, b_i]} |g_i(x_i)| \leq \sup_{x_i \in [a_i, b_i]} |x_i^{\gamma_i}| +\epsilon \leq ((\max\{|a_i|, |b_i|\})^{\gamma_i}   + \epsilon \leq (\max\{|a_i|, |b_i|\})^{\gamma_i} + 1 = C_i.
\]

According to~\cite{hardy1952inequalities}, for all $a_i, b_i \in \R$ such that $|a_i|\leq 1$ and $|b_i| \leq 1$, we have that $\left|\prod_{i=1}^n a_i- \prod_{i=1}^n b_i \right| \leq \sum_{i=1}^n |a_i - b_i|$. Let $g(x)=\prod_{i=1}^n g_i(x_i)$. Then, since $\left|\frac{g_i(x_i)}{C_i}\right| \leq 1$, $\left|\frac{x_i^{\gamma_i}}{C_i}\right| \leq 1$ and $C_i>1$, we have that
\begin{align*}
\sup_{x \in X}\left| x^{\gamma}-  g(x)\right|&=\sup_{x \in X}\left|\prod_{i=1}^n x_i^{\gamma_i}- \prod_{i=1}^n g_i(x_i)\right| \\
&=C \sup_{x \in X}\left|\prod_{i=1}^n \frac{x_i^{\gamma_i}}{C_i}- \prod_{i=1}^n \frac{g_i(x_i)}{C_i}\right| \\
 &\leq C \sum_{i = 1}^n \left|\frac{x_i^{\gamma_i}}{C_i}-  \frac{g_i(x_i)}{C_i}\right| \\
 &\leq n C \max_{i\in \{1,\cdots,n\}}  |x_i^{\gamma_i}-  g_i(x_i)| \leq  n C \epsilon.
\end{align*}

We conclude that for any $\gamma \in \N^n$, $\lambda_\gamma\in \R$ and $\varepsilon>0$, there exists $\{\alpha_{\gamma, j}\}_{j = 1}^{m_\gamma}$ and $\{y_{\gamma, j}\}_{j = 1}^{m_\gamma}$ such that
\[
\sup_{x\in X} \normmmmm{\sum_{j = 1}^{m_\gamma} \alpha_{\gamma, j} k(x, y_{\gamma, j}) - \lambda_\gamma x^{\gamma}} \leq \varepsilon.
\]

Now, the Weierstrass approximation theorem~\citep{willard2012general} states that polynomials as a linear combination of monomials are dense in $\mcl C(X)$. Thus for any $f \in \mcl C (X)$ and $\varepsilon > 0$ there exists a maximal degree $d\in \N$ and a polynomial function $p(x) = \sum_{\|\gamma\|_1 \leq d } \lambda_\gamma x^{\gamma}$ such that
\[
\sup_{x \in  X} \normmm{f(x) - p(x)}   \leq \varepsilon.
\]
Let $\{\alpha_{\gamma, j}\}_{j = 1}^{m_\gamma}$ and $\{y_{\gamma, j}\}_{j = 1}^{m_\gamma}$ be as defined above. Then, using the triangle inequality, we have
\begin{align*}
\sup_{x \in X} \normmmmm{f(x) - \sum_{\|\gamma\|_1 \leq d }  \sum_{j = 1}^{m_\gamma} \alpha_{\gamma, j} k(x, y_{\gamma, j})} &= \sup_{x \in X} \normmmmm{f(x) - p(x) + p(x) -\sum_{\|\gamma\|_1 \leq d }  \sum_{j = 1}^{m_\gamma} \alpha_{\gamma, j} k(x, y_{\gamma, j})} \\
&\hspace{-10mm} \leq \sup_{x \in X} \normmm{f(x) - p(x)} + \sup_{x \in X} \normmmmm{p(x) - \sum_{\|\gamma\|_1 \leq d }  \sum_{j = 1}^{m_\gamma} \alpha_{\gamma, j} k(x, y_{\gamma, j}) }  \\
&\hspace{-10mm} \leq  \sup_{x \in X} \normmm{f(x) - p(x) } + \sum_{\|\gamma\|_1 \leq d } \sup_{x \in X} \normmmmm{\lambda_\gamma x^{\gamma} - \sum_{j = 1}^{m_\gamma} \alpha_{\gamma, j} k(x, y_{\gamma, j}) } \\
                                                           & \hspace{-10mm}\leq  \left(\begin{pmatrix}n+d\\d\end{pmatrix}+1\right)\varepsilon.
\end{align*}

Thus, since for any $f \in \mcl C(X)$ and any $\epsilon>0$, there exists $g =  \sum_{\|\gamma\|_1 \leq d }  \sum_{j = 1}^{m_\gamma} \alpha_{\gamma, j} k(x, y_{\gamma, j})$ such that $\norm{f - g}_C \leq \epsilon$. Therefore, since any $k \in \mcl K_T$ is continuous \citep[See][]{JMLR}, we have that $k$ is universal.
\end{proof}
\section{An Accelerated Algorithm for Quadratic Convergence}\label{appendix:APD}
In this appendix, we consider the Accelerated Primal-Dual (APD) algorithm discussed in the main text which can be used to achieve quadratic convergence for Generalized Kernel Learning. First, we define the APD algorithm and prove quadratic convergence. Furthermore, we show that the APD algorithm can be decomposed into primal and dual sub-problems which may be solved using QP and the SVD in a manner similar to the FW algorithm. Finally, we note that the APD algorithm underperforms the proposed FW algorithm for the first several iterations and hence we propose a hybrid algorithm which uses FW until an error tolerance is satisfied and then switched to APD for subsequent iterations. %Finally, we show that the resulting algorithm provides fast initial convergence and a quadratic upper bound at subsequent iterations.

\subsection{An Algorithm with $O\left(\frac{1}{k^2}\right)$ Convergence}

As discussed in Section~\ref{sec:Ideal}, the Generalized Kernel Learning problem can be represented as a minmax or saddle point optimization problem~\eqref{eqn:OPT} which is linear in $P$ and convex in $\alpha$
\begin{equation}
	\min_{\alpha\in\mathcal{Y}} \max_{P\in\mathcal{X}} f(\alpha) + \Phi(\alpha) - h(P),
	\label{eq:APD_FORM}
\end{equation}
where for Generalized Kernel learning we have that $f(\alpha) = \kappa_{\star}(\alpha)$, $\Phi(\alpha) = \lambda(e_{\star} \odot \alpha)$, and $h(P) = 0$ as defined in section~\ref{sec:2step}.% Here, we use for classification $\star=c$ and for regression $\star=r$.  In classification case $e_c = y$ and for regression $e_r = \mathbf{1}^m$ is the vector of ones.

Numerically, we observe that for several first iterations, the Frank-Wolfe GKL algorithm achieves super-sublinear convergence - which is often sufficient to achieve an error tolerance of $10^{-5}$. However, if lower error tolerances are desired, the sublinear convergence rate of the FW algorithm at higher iterations can be accelerated by a algorithm with $O\left(\frac{1}{k^2}\right)$ convergence.

Fortunately, \cite{Hamedani2020APA} has shown that provable $O\left(\frac{1}{k^2}\right)$ convergence can be achieved using a variation of an algorithm originally proposed in~\cite{chambolle2016ergodic}. This algorithm, the accelerated primal-dual (APD) algorithm, requires computation of the same sub-problems, $OPT\_A$ and $OPT\_P$, but it achieves the worst case $O\left(\frac{1}{k^2}\right)$ convergence.

Specifically, this algorithm can be used to solve problems of the form
\begin{equation}
\min_{\alpha\in\mathcal{Y}_\star} \max_{P\in\mathcal{X}} -\kappa_{\star}(\alpha) - O(\alpha \odot e_{\star}, P) = f(\alpha) + \Phi(P,\alpha) - h(P),
\label{eq:APD_FORM}
\end{equation}
where $f(\alpha) + \Phi(\alpha)$ is strongly convex and $h(P)$ is concave. Since the Generalized Kernel Learning Problem (GKL~\eqref{KernelSVC_P}) has the same form as problem~\eqref{eq:APD_FORM}, application of this approach is relatively straightforward. Specifically, consider Algorithm~\ref{APD} from~\cite{Hamedani2020APA} which requires the following definition.
\begin{defn}[Bregman distance~\citep{Hamedani2020APA}]
Given $f$ and $h$, let $\phi_Y : \mathcal{Y}\rightarrow \R$ and $\phi_X : \mathcal{X}\rightarrow \R$ be differentiable functions on open sets $\mcl Y \subset \textbf{dom}\,f$ and $\mcl X \subset \textbf{dom}\, h$. Suppose $\phi_X$ and $\phi_Y$ have
closed domains and are 1-strongly convex with respect to $\|\cdot\|_{\mcl X}$ and $\|\cdot\|_{\mcl Y}$, respectively. We define the Bregman distances $D_{\mcl X}:\mathcal{X}\times\mathcal{X}\rightarrow \R$ and $D_{\mcl Y}:\mathcal{Y}\times\mathcal{Y}\rightarrow \R$ corresponding to $\phi_X$ and $\phi_Y$, respectively, as $D_{\mcl X}(x, \hat{x}) = \phi_{\mcl X}(x) - \phi_{\mcl X}(\hat{x}) - \left<\nabla \phi_{\mcl X}(\hat{x}), x - \hat{x}\right>$ and $D_{\mcl Y}(y, \hat{y}) = \phi_{\mcl{Y}}(y) - \phi_{\mcl{Y}}(\hat{y}) - \left<\nabla \phi_{\mcl Y}(\hat{y}), y - \hat{y}\right>$.
\end{defn}

 \begin{algorithm}[H]
\begin{algorithmic}
%\STATE \texttt{Given} $t \in [0,1]$;
\STATE \texttt{Initialize} $\mu, \tau_0, \sigma_0$, $k_{\text{max}}$ and $\alpha_0, P_0 \in \mathcal{Y}\times \mathcal{X}$\\
$k=0$, $(\tau_{-1},  \sigma_{-1}) = (\tau_0, \sigma_0)$,  $(P_{-1},  \alpha_{-1}) = (P_0, \alpha_0)$ and $\gamma_0 = \frac{\sigma_0}{\tau_0}$;
\WHILE{$k< k_{\text{max}}$}

\STATE \text{1:} $\displaystyle \sigma_k = \gamma_k \tau_k,\quad \theta_k = \sigma_{k-1}/\sigma_{k}$
\STATE \text{2:} $x_k = \nabla_P \Phi(\alpha_k, P_k)$
\STATE \text{3:} $S = (1+\theta_k) x_k - \theta_k x_{k-1}$
\STATE \text{4:} $\displaystyle P_{k+1} = \arg\min_{P\in\mathcal{X}} \frac{1}{\sigma_k} D_X(P, P_k) - \left<S, P\right> $\\
\STATE \text{5:} $\displaystyle \alpha_{k+1}  =  \arg\min_{\alpha\in\mathcal{Y}} f(\alpha)  +  \Phi(P_{k+1},\alpha)  +  \frac{D_Y(\alpha, \alpha_k)}{\tau_k}$

\STATE \text{6:}  $\gamma_{k+1} = \gamma_k(1+\mu \tau_k)$, $\tau_{k+1} = \tau_k\sqrt{\frac{\gamma_k}{\gamma_{k+1}}}$, $k = k+1$
\ENDWHILE
\end{algorithmic}
\caption{APD algorithm} \label{APD}
\end{algorithm}
%The following definition and theorem can be found in~\cite{Hamedani2020APA}.
Theorem~\ref{thm:hamedani} proves $O\left(\frac{1}{k^2}\right)$ convergence of Algorithm~\ref{APD} when $f(\alpha) + \Phi(P, \alpha)$ is strongly convex in $\alpha$ and $h(P)$ is concave.
% If $f$ is strongly convex, we can fix $D_X(x,\hat{x)} = \frac{1}{2}\|x-\hat{x}\|_X^2$
% \setcounter{thm}{15}
\begin{thm}[\cite{Hamedani2020APA}]\label{thm:hamedani}
Let $D_\mathcal{X}$ and $D_\mathcal{Y}$ be Bregman distance
functions. Suppose that for any $P\in \mcl X$, $f(\alpha)$ and $\Phi(P,\alpha)$ are convex in $\alpha$ and for any $\alpha \in \mcl Y$, $\Phi(P, \alpha)$ and $-h(P)$ are concave in $P$. In addition, suppose $f$ is strongly convex with modulus $\mu>0$.
Furthermore, suppose that $L_{\alpha\alpha},L_{P\alpha}$ satisfy
\begin{align*}
     \|\nabla_P\Phi(P,\alpha)-\nabla_P\Phi(\hat{P},\hat{\alpha})\|_\mathcal{X^*} &
      \leq L_{P\alpha} \|\alpha - \hat{\alpha}\|_{\mathcal{Y}}  \\
     \|\nabla_\alpha\Phi(P,\alpha)-\nabla_\alpha\Phi(P,\hat{\alpha})\|_\mathcal{Y^*} &
      \leq L_{\alpha\alpha} \|\alpha - \hat{\alpha}\|_{\mathcal{Y}}
\end{align*}
for all $\alpha,\hat{\alpha} \in \mcl Y$ and $P,\hat P \in \mcl X$ and that the starting parameters $\tau_0$ and $\sigma_0$ satisfy:
\begin{align*}
     \left(\frac{1-\delta}{\tau_0} - L_{\alpha\alpha}\right)\frac{1}{\sigma_0} & \geq \frac{L_{P\alpha}^2}{c_a}
\end{align*}
for some $\delta, c_a \in \R_+$, such that $c_a  + \delta \leq 1$.   If
\begin{equation*}
\{\alpha^*,P^*\}=\arg\min_{\alpha\in\mathcal{Y}_\star} \max_{P\in\mathcal{X}} f(\alpha) + \Phi(P,\alpha) - h(P),
\end{equation*}
exists, then for any sequence produced by Algorithm~\ref{APD}, $\{\alpha_k, P_k\}$, we have that\\
1. $\lim_{k \rightarrow \infty}\{ P_k, \alpha_k\}\rightarrow \{P^*, \alpha^*\}$ \\
2. If $\delta>0$,  and we define $ L(\alpha,P):=f(\alpha) + \Phi(P,\alpha) - h(P)$ then
\begin{equation*}
    0\leq L(P^*, \alpha_k) - L(P_k, \alpha^*) \leq O\left(\frac{1}{k^2}\right).
\end{equation*}
\label{theorem:APD}
\end{thm}
\begin{proof}
See~\cite{Hamedani2020APA}
\end{proof}
\textit{Remark:} As stated in~\cite{Hamedani2020APA},  the conditions of Theorem~\ref{thm:hamedani} are satisfied using $\tau_0 = \frac{1}{3L_{\alpha\alpha}}$ and  $\sigma_0 = \frac{L_{\alpha\alpha}}{2L_{P\alpha}^2}$.

\subsection{Proposed Booster Algorithm}
In this subsection, we propose the following algorithm applicable to our optimization problem. This algorithm is a specification of APD for the Generalized Kernel learning.

 \begin{algorithm}[H]
\begin{algorithmic}
%\STATE \texttt{Given} $t \in [0,1]$;
\STATE \texttt{Initialize} $(\alpha_0, P_0) \in \mcl Y_{\star} \times \mcl X$, $\mu$ \\
$\sigma_{-1}=\sigma_0 = \frac{L_{\alpha\alpha}}{2L_{\alpha P}^2}$, $\tau_{-1}=\tau_0 = \frac{1}{3L_{\alpha\alpha}}$\\
$k=0$, $(P_{-1},  \alpha_{-1}) = (P_0, \alpha_0)$ and $\gamma_0 = \frac{\sigma_0}{\tau_0}$;
% \WHILE{$APD\_P(\arg APD\_A(P_k, S, \sigma_k),P_{k+1}, \tau_k )-APD\_A(P_k, S, \sigma_k)>\epsilon$}
\WHILE{$ L( P_{k+1}, \alpha_k) -  L(P_k, \alpha_{k+1})>\epsilon$}

\STATE \text{1:} $\displaystyle \sigma_k = \gamma_k \tau_k,\quad \theta_k = \sigma_{k-1}/\sigma_{k}$
%\STATE \text{2:} $x_k = \frac{1}{2} \alpha^T \text{diag}(e_{\star}) G  \text{diag}(e_{\star}) \alpha$
\STATE \text{2:} $x_k = \frac{1}{2} D(e_{\star} \odot \alpha_k)$
\STATE \text{3:} $S = (1+\theta_k) x_k - \theta_{k} x_{k-1}$
\STATE \text{4:} $\displaystyle P_{k+1} = \arg APD\_P(P_k, S, \sigma_k)$
% $\arg\min_{P\in\mathcal{X}} \frac{1}{\sigma_k} D_P(P, P_k) - <s, P> $\\
\STATE \text{5:} $\displaystyle \alpha_{k+1} = \arg APD\_A(\alpha_k, P_{k+1}, \tau_k)$
% $\arg\min_{\alpha\in\mathcal{Y}} 2e^T\alpha + \Phi(P_{k+1}, , \alpha) + \frac{1}{\tau_k} D_\alpha(\alpha, \alpha_k))$

\STATE \text{6:}  $\gamma_{k+1} = \gamma_k(1+\mu \tau_k)$, $\tau_{k+1} = \tau_k\sqrt{\frac{\gamma_k}{\gamma_{k+1}}}$, $k = k+1$
\ENDWHILE
\end{algorithmic}
\caption{APD algorithm} \label{APDour}
\end{algorithm}

where  $ L(P, \alpha):=-\lambda(e_{\star} \odot \alpha, P)-\kappa_{\star}(\alpha)$ and
\[D_{i,j}(\alpha):= \sum_{k,l=1}^m \alpha_k e_{*k} G_{i,j}(x_k,x_l) \alpha_l e_{*l}\]
is defined in Eqn.~(12) and the two subroutines $APD\_P$ and $APD\_A$ are defined as
\begin{align*}
  APD\_A(P, \alpha_k,  \tau)&:=  \max_{\alpha\in \mcl Y_{\star}} \lambda(e_{\star} \odot \alpha, P)+\kappa_{\star}(\alpha) -\frac{1}{\tau} D_{\mcl Y}(\alpha, \alpha_k),\\
APD\_P(P_{k}, S, \sigma)&:=  \min_{P\in \mcl X} \frac{1}{\sigma} D_{\mcl X}(P, P_k) - \left< S, P \right>
\end{align*}
where $D_{\mcl Y_{\star}}:=\frac{1}{2}\|\cdot\|_2^2$, and $D_{\mcl X}:=\frac{1}{2}\|\cdot\|_F^2$.

Furthermore,
\begin{equation}
\label{eqn:tau0}
\tau_0 = \frac{1}{3L_{\alpha\alpha}}, \;\; \text{and} \;\;\sigma_0 = \frac{L_{\alpha\alpha}}{2L_{\alpha P}^2}
\end{equation}
where
\begin{equation}\label{eqn:L}
L_{\alpha\alpha}:=\frac{n_P}{2}\sum_{ij}|D_{ij}(e_{\star})|^2,~~ L_{\alpha P}:=\frac{C}{n_P} L_{\alpha\alpha}
\end{equation}
where recall  that $n_P$ is determined by the size of $P$ ($P\in \S^{n_P}$) and $C>0$ can be chosen arbitrarily. Finally, we choose $\mu>0$ sufficiently small such that $-\lambda(\alpha\odot e_{\star}, P) - \mu\alpha^T \alpha$ is convex for all $P\in \mcl X$.

\subsection{$O\left(\frac{1}{k^2}\right)$ Convergence Proof for Algorithm~\ref{APDour}}
Formally, we state the theorem.
\begin{thm} \label{thm:convergence}
Algorithm~\ref{APDour} returns iterates $P_k$ and $\alpha_k$ such that,
$L(\alpha_k,P_{k+1})-L(\alpha_{k+1}, P_k) < O(\frac{1}{k^2})$.
\end{thm}
\begin{proof}
In this proof, we first show that Algorithm~\ref{APDour} returns iterates $P_k$ and $\alpha_k$ which satisfy Algorithm~\ref{APD}. Next, we show that if $\tau_0,\sigma_0$ are chosen as per Eqn.~\eqref{eqn:tau0}, then the conditions of Theorem~\ref{theorem:APD} are satisfied. First, let us define
\begin{align}
 \label{opt:theorem_eqn1}
   \notag f(\alpha) &:= -\kappa_{\star}(\alpha) + \mu \alpha^T \alpha \\
    \Phi(P, \alpha) &:= -\lambda(\alpha\odot e_{\star}, P) - \mu\alpha^T \alpha \\
    \notag h(P) &:= 0 \\[-7mm] \notag
\end{align}
for some sufficiently small $\mu>0$.
Now, suppose that $\{P_k,\alpha_k,\gamma_k,S_k,x_k,\sigma_k,\tau_k\}$ satisfy Algorithm~\ref{APDour}. Clearly, these iterations also satisfy Steps 1, 3 and 6 of Algorithm~\ref{APD}. Furthermore, these iterations satisfy the equation defined in Step 2 since
\begin{align*}
\nabla_P & \Phi(\alpha, P)_{ij} = \frac{\partial}{\partial P_{ij}}\left[ -\lambda(\alpha\odot e_{\star}, P) - \mu\alpha^T\alpha \right]\\
 &= \frac{\partial}{\partial P_{ij}} \left[\frac{1}{2}\sum_{i,j=1}^{n_P} P_{ij}  \sum_{k,l=1}^m (\alpha_k e_{*k}) G_{i,j}(x_k,x_l) (\alpha_l e_{*l})\right] \\
 &= \frac{1}{2}\sum_{k,l=1}^m (\alpha_k e_{*k}) G_{i,j}(x_k,x_l) (\alpha_l e_{*l}) = \frac{1}{2} D_{ij}(e_{\star} \odot \alpha_k).
\end{align*}

Next, the proposed iterations satisfy the equation defined in Step 4 of Algorithm~\ref{APD}  since by the definition of $APD\_P$
\begin{align*} P_{k+1} &= \arg APD\_P(P_{k}, S, \sigma_k)  =\arg \min_{P\in \mcl X} \frac{1}{\sigma_k} D_{\mcl X}(P, P_k) - \left< S, P \right>.
\end{align*}

Finally, the equality in Step 5 of Algorithm~\ref{APD} is satisfied  since
\begin{align*}
&\alpha_{k+1} = \arg APD\_A(P_{k+1}, \alpha_k,  \tau_k) \\
&~~ = \arg \max_{\alpha\in \mcl Y_{\star}} \lambda(e_{\star} \odot \alpha, P_{k+1})+\kappa_{\star}(\alpha) -\frac{1}{\tau_k} D_{\mcl Y}(\alpha, \alpha_k)\\
&~~ = \arg \min_{\alpha\in \mcl Y_{\star}} -\lambda(e_{\star} \odot \alpha, P_{k+1})-\kappa_{\star}(\alpha) +\frac{1}{\tau_k} D_{\mcl Y}(\alpha, \alpha_k)\\
&~~ = \arg\min_{\alpha\in\mathcal{Y_{\star}}} f(\alpha)  +  \Phi(P_{k+1},\alpha)  +  \frac{1}{\tau_k}D_Y(\alpha, \alpha_k).
\end{align*}

Therefore, we have that $\{P_k,\alpha_k,\gamma_k,S_k,x_k,\sigma_k,\tau_k\}$ satisfy Algorithm~\ref{APD}.

We next must show that $\Phi$ is concave in $P$, convex in $\alpha$ and $f$ is strongly convex.  $\Phi$ is linear in $P$ and thus concave in $P$. As defined, $\Phi$ is convex in $\alpha$ and clearly, $f$ is strongly convex for any $\mu>0$. Since $L_{\alpha\alpha}, L_{\alpha P}$ are as defined in Equation~\eqref{eqn:L}, then
\begin{align*}
     \|\nabla_P\Phi(P,\alpha)-\nabla_P\Phi(\hat{P},\hat{\alpha})\|_\mathcal{X^*} &
      \leq L_{P\alpha} \|\alpha - \hat{\alpha}\|_{\mathcal{Y}}  \\
     \|\nabla_\alpha\Phi(P,\alpha)-\nabla_\alpha\Phi(P,\hat{\alpha})\|_\mathcal{Y^*} &
      \leq L_{\alpha\alpha} \|\alpha - \hat{\alpha}\|_{\mathcal{Y}}
\end{align*}
as desired. Finally, we have that if $\tau_0$ and $\sigma_0$ are as defined in Equation~\eqref{eqn:tau0}, $\delta=\frac{1}{4}$, and $c_a=\frac{1}{2}$ , then $c_a+\delta \le 1$ and
\begin{align*}
     \left(\frac{1-\delta}{\tau_0} - L_{\alpha\alpha}\right)\frac{1}{\sigma_0} & \geq \frac{L_{P\alpha}^2}{c_a}
\end{align*}
as desired.

Therefore, we have by Theorem~\ref{theorem:APD} that
\begin{equation*}
    0\leq L(P^*, \alpha_k) - L(P_k, \alpha^*) \leq O\left(\frac{1}{k^2}\right),
\end{equation*}
and hence
\begin{align*}
L(\alpha_k,P_{k+1})-L(\alpha_{k+1}, P_k) &<L(P^*, \alpha_k) - L(P_k, \alpha^*)  \le  O\left(\frac{1}{k^2}\right).
\end{align*}
\end{proof}
We next define efficient algorithms to solve the subroutines $APD\_P$ and $APD\_D$.

\subsection{Solving \textit{APD\_P}}
The fourth step of the APD algorithm requires solving $APD\_P$.  For arbitrary matrix $S \in \S^{n_P}$ this optimization problem is formulated as follows.
\begin{align} \label{opt:APDP}
P_{k+1} = \arg\min_{P\in\mathcal{Y}} \;\;\frac{1}{\sigma_k} D_X(P, P_k) - \left< S,P \right> \end{align}
The following algorithm solves the optimization task~\ref{opt:APDP}\newline

 \begin{algorithm}[H]
	\begin{algorithmic}
		%\STATE \texttt{Given} $t \in [0,1]$;
		\STATE \texttt{Input} $P_k, S\in \S^{n_P}$, $\sigma_k, \varepsilon>0$ ;
		\STATE \texttt{Set} $r = +\infty$ and $A = P_k + \sigma_k S$,\\
		\STATE \texttt{Singular Value Decomposition:} $A = \sum_i \lambda_i p_i p_i^T$
		\STATE \texttt{Initialize} $y_l = \min_{i} \lambda_i$ and $y_u = \max_i \lambda_i$;
		\WHILE{$r>\varepsilon$}
		
		\STATE \text{1:}~~ $y = \frac{1}{2}(y_l + y_u)$
		\STATE \text{2:}~~ $r = \sum_i |\lambda_i - y|_+ -n_P$
		\STATE \text{3:}~~ update $y_u = y$ if $r\geq 0$, or $y_l = y$ otherwise
		\ENDWHILE
		\STATE \texttt{Return}
		$P = \sum_{i}|\lambda_i - y|_+ p_i p_i^T=\arg APD\_P(P_k,S,\sigma).$
	\end{algorithmic}
	\caption{$APD\_P$ Subroutine} \label{APD_SOL_P}
\end{algorithm}
\begin{lem}
	Let the optimization problem be as defined in~\eqref{opt:APDP}, then the algorithm~\ref{APD_SOL_P} returns the solution of \textit{APD\_P}.
\end{lem}
\begin{proof}
	Firstly, we should  reformulate Optimization Problem~\eqref{opt:APDP} as,
	\begin{align}
		&\notag \arg  \min_{\substack{P \in \S^{n_P} \\ \text{trace}(P) = n_P \\ P \succ 0}}    \frac{1}{\sigma} D_X(P, P_k) - \left< S, P \right>  =   \arg \min_{\substack{P \in \S^{n_P}\\ \text{trace}(P) =n_P \\ P \succ 0}}   \frac{1}{2\sigma} \|P - P_k\|_F^2 - \ip{S}{P}\\ \notag
		%&= \arg \min_{\substack{P \in \R^{q \times q}\\ \text{trace}(P) = q \\ P \succ 0}} % \frac{1}{2\sigma} \text{trace}(P^TP - 2P^TP_k + P_k^TP_k - 2\sigma S^T P)\\
		&= \arg \min_{\substack{P \in \S^{n_P}\\ \text{trace}(P) = n_P \\ P \succ 0}}  \frac{1}{2\sigma} \ip{P-P_k}{P-P_k} - 2\sigma \ip{S}{P} = \arg \min_{\substack{P \in \S^{n_P}\\ \text{trace}(P) = n_P \\ P \succ 0}} \frac{1}{2\sigma} \|P - (P_k+\sigma S)\|_F^2, \notag \\
		&\qquad = \arg \min_{\substack{P \in \S^{n_P}\\ \text{trace}(P) = n_P\\ P \succ 0}}  \|P - A\|_F^2,\label{opt_FNO}
	\end{align}
where $A = P_k + \sigma S$.

Having reduced $APD\_P$ to a convex distance minimization problems of the form of Eqn.~\eqref{opt_FNO}. According to~\cite{Harada2018PositiveSM}, we use a subroutine defined by Algorithm~\ref{APD_SOL_P} to solve problem~\eqref{opt_FNO} and therefore to find $P_{k+1}$ in Step 4 of Algorithm~\ref{APDour}.
\end{proof}

\subsection{Solving \textit{APD\_A}}
$APD\_A$ is a QP of the form
\begin{align*} %\label{opt:1}
  \min_{\substack{\alpha \in \mcl Y_{\star} }}\;\;   \frac{1}{2  }\alpha^T Q_* \alpha + c_*^T \alpha,
\end{align*}
where, $c_*^T \alpha = -\tau \kappa_{\star}(\alpha) - \alpha_p^T \alpha,$ and $\alpha^T Q_* \alpha = -\tau \lambda(e_{\star}\odot\alpha, P) + \alpha^T\alpha$.
% $c_*^T \alpha= -\tau \kappa_{\star}(\alpha) - \alpha_p^T \alpha$ and
% $\alpha^T Q \alpha = -\tau O(e_{\star}\odot\alpha, P) + \alpha^T\alpha$
% $\tau \text{\text{diag}}(e_{\star}) K \text{diag}(e_{\star}) + I$

QP's of this form can be solved using a slight variation of the algorithm proposed in Subsection~\ref{subsec:step1}.

\subsection{Combined Solution for General Kernel Learning}
We can see that the Frank-Wolfe GKL Algorithm converges quite quickly up to a certain value, but after that the convergence slows down and becomes linear. Moreover, APD algorithm return a smaller objective function after 3000-4000 iterations. But, the pure APD has non-monotonic convergence at the early stage. All this together prompted us to create a combined Frank-Wolfe GKL and APD algorithm. The tolerance for the Frank-Wolfe GKL was chosen according to the numerical results.

\begin{algorithm}[H]
\begin{algorithmic}
%\STATE \texttt{Given} $t \in [0,1]$;
\STATE \texttt{INPUT} $\epsilon$ - tolerance;
\STATE \text{1:} $(\alpha_1, P_1) =  \text{Frank Wolfe GKL with tolerance } \epsilon$
\STATE \text{2:} $(\alpha_2, P_2) = \text{APD with initial guess } \alpha_1, P_1 \text{ and desired tolerance}$
\STATE \text{3:}~~ $\alpha = OPT\_A(P_2)$, $P = P_2$
\STATE \texttt{OUTPUT:}~ $\alpha$, $P$
\end{algorithmic}
\caption{Final version of GKL} \label{finalTKL}
\end{algorithm}

We assume, that the proposed algorithm~\ref{finalTKL} will show both fast initial convergence and $O\left(\frac{1}{k^2}\right)$ convergence in the worst case, which is necessary for application to arbitrary data sets.

\section{Computational Complexity of SVM Problem with Optimal Kernel}\label{sec:complexity_SVM}
In this section, we consider the computational complexity of SVM subproblem as a function of matrix $P$. Although, as discussed in Subsection~\ref{subsec:complexity}, the computation complexity of SVM subproblem does not depend on the size of matrix $P$, it implicitly depends on the kernel function. Note, that the proposed solution demonstrates that the optimal matrix $P$ is always a rank $1$ matrix -- See Section~\ref{sec:2step}. However, the proposed set of kernels includes many different kernels -- e.g. for matrix $P$ with different ranks. The computational complexity of SVM subproblem is uniquely determined by kernel function and therefore, in our case, depends on matrix $P$.  To investigate this issue, we generate random positive semidefinite matrices with different ranks and consider the number of iterations of SVM learning problem. We compute matrix $P$ randomly
\[
P = \frac{n_P}{r}\sum_{i = 1}^{r} \frac{1}{\|v_i\|_2} v_i v_i^T
\]
where $v$ is a normal distributed vector and $r$ is a desired rank of matrix $P$.

 In Figure~\ref{fig:SVM_complexity} we plot the number of iterations required to achieve the fixed tolerance $\varepsilon = 0.1$ for different ranks of matrix $P$. The data set used for these plots is California Housing (CA) in~\cite{pace1997sparse}. The results shows that the rank $1$ optimal solution of the proposed algorithm is significantly faster in comparison with other random positive semidefinite matrices $P$ with different ranks.

\begin{figure*}

\centering
    \begin{subfigure}[t]{0.45\textwidth}
        \centering
\includegraphics[width =1\textwidth]{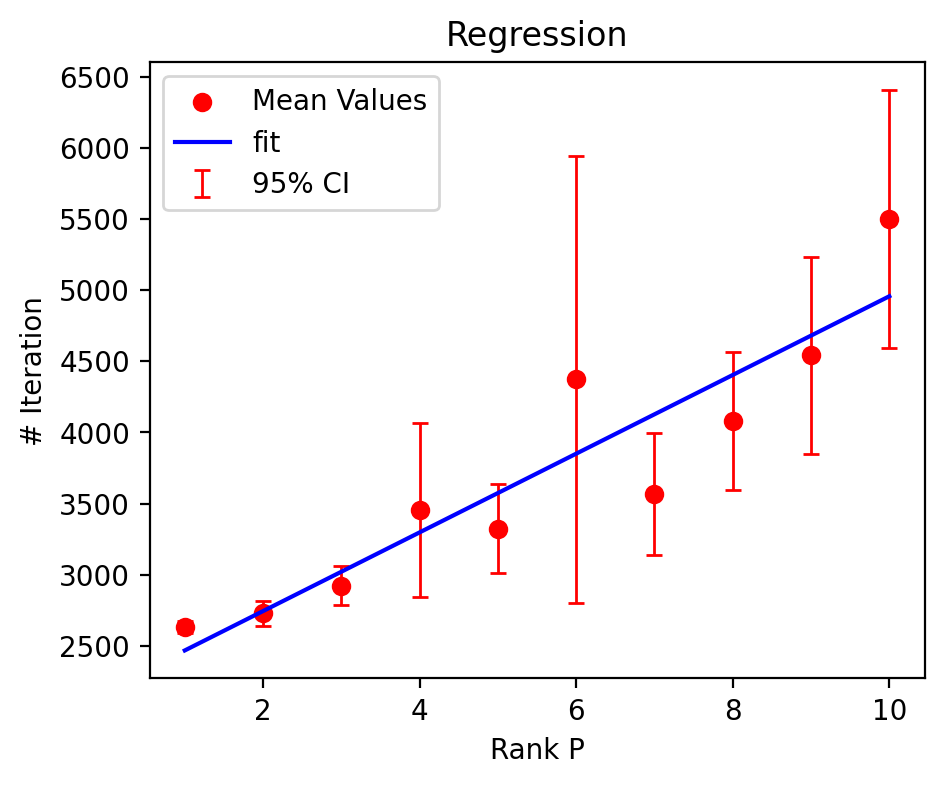}
        \caption{The number of iteration for random matrix $P$ for Support Vector Regression.}
       % \caption{ FW classification, iteration complexity}
    \end{subfigure}%
    ~
    \begin{subfigure}[t]{0.45\textwidth}
        \centering
\includegraphics[width =1\textwidth]{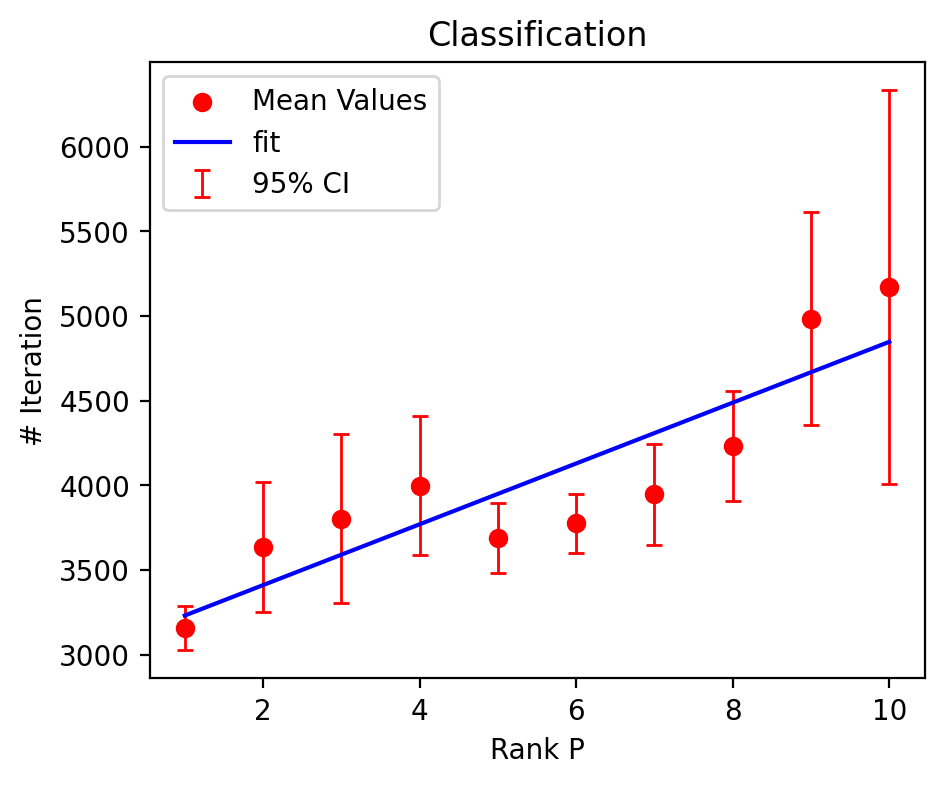}

        \caption{The number of iteration for random matrix $P$ for Support Vector Classification.}
       %\caption{APD classification, iteration complexity}
    \end{subfigure}
    \caption{  \small The number of iterations of SVM subproblem required to achieve the desired tolerance $\varepsilon=0.1$ as a function of the rank $P$. The SVM subproblem has been solved using LibSVM implementation. The red dots and error bars represent average number of iterations of the SVM algorithm and 95$\%$ confidence interval using 20 trials for a) regression problem for California Housing (CA) data set in~\cite{pace1997sparse} and b) for classification problem for Shill Bid data set in~\cite{alzahrani2018scraping,alzahrani2020clustering}. We also included the blue line, that indicates the best linear fit of the average number of iterations.  }\label{fig:SVM_complexity}
\end{figure*} 

\section{Implementation and Documentation of Algorithms}
In this appendix material, we have provided a MATLAB implementation of the proposed algorithms which can be used to reproduce the numerical results given in Section~\ref{sec:scalability}. The primary executable is \texttt{PMKL.m}. The demo files \texttt{exampleClassification.m} and \texttt{exampleRegression.m} illustrate typical usage of this executable for classification and regression problems respectively. This software is available from Github~\citep{githubTKL}.

\subsection{Documentation for Included Software} \label{sec:Guidance}

Also included in the main material are the 5 train and test partitions used for each of the 12 data sets used in the numerical results section of the paper. The code \texttt{numericalTest.m} allows the user to select the data set and run the FW PMKL algorithm on the five partitions to calculate the average and standard deviation of the MSE (for regression) or TSA (for classification).

\textbf{The PMKL subroutine}\\
PMKL\footnote{\texttt{PMKL\_Boosted} is the combined Frank-Wolfe and Accelerated Primal Dual method that is used when high accuracy is required.  The algorithm usage is identical to the PMKL algorithm and can be used with these same instructions.} - Positive Matrix Kernel Learning,
\begin{verbatim}
>> f = PMKL(x,y,Type,C,params);
\end{verbatim}
yields an optimal solution to the minimax program
\[ \min_{P \in \mcl X} \max_{\alpha \in \mcl Y_{\star}} \lambda(e_{\star} \odot \alpha, P) + \kappa_{\star}(\alpha), \]
where $\mcl X:=\{P \in \S^{n_P}\;:\; \text{trace}(P) = n_P,\; P \succeq 0\}$,
\begin{align*}
\mcl Y_c:=\{\alpha \in \R^m\; : \; \alpha^T y = 0,\; \alpha_i \in [0,C] \}, \quad \mcl Y_r:=\{\alpha \in \R^m\;:\; \alpha^T e = 0, \; \alpha_i \in [-C, C]\},
\end{align*}
and where,
\begin{itemize}
    \item $\texttt{x} \in \R^{n \times m}$ is a matrix of $n$ rows corresponding to the number of features and $m$ columns corresponding to the number of samples where $x(:,i)$ is the i'th sample,
    \item  $\texttt{y} \in \R^m$ ($\texttt{y} \in [-1,1]^m$) is a row of outputs (labels) for each of the samples in $\texttt{x}$ where $\texttt{y(:,i)}$ is the i'th output corresponding to the i'th sample $\texttt{x(:,i)}$,
    \item \texttt{Type} is the string \texttt{'Classification'} for classification or \texttt{'Regression'} for regression,
    \item \texttt{C} is the penalty for miss-classifying a point for classification, or for predicting an output with less than $\epsilon$ accuracy for regression,
    \item \texttt{params} is a structure containing additional optional parameters such as the kernel function (default is the TK function), kernel parameters (degree of monomial basis for TK or GPK kernels), the domain of integration, the $\epsilon$-loss term for regression, the maximum number of iterations and the tolerance,
    \item The output \texttt{f} is an internal data structure containing the solutions $P^*$ and $\alpha^*$, as well as the other user selected parameters. This data structure defines the resulting regressor/classifier. This regressor/classifier can be evaluated using the \texttt{EvaluatePMKL} command as described below.
\end{itemize}
Default parameters of the params structure are
\begin{verbatim}
>> params.kernel  = 'TK'
>> params.delta  = .5
>> params.epsilon = .1
>> params.maxit = 100
>> params.tol = .01
\end{verbatim}
where \texttt{kernel} specifies the kernel function to use, \texttt{delta} determines the bounds of integration $[a,b] = [0-\delta,1+\delta]^n$, \texttt{epsilon} is the epsilon-loss of the support vector regression problem, \texttt{maxit} is the maximum number of iterations, and \texttt{tol} is the stopping tolerance.  The \texttt{PMKL.m} function can be run with only some of the inputs manually specified, as discussed next.

\textbf{Default Implementation}\\
To run the PMKL algorithm with all default values for the samples $\texttt{x}$ and outputs $\texttt{y}$ and to automatically select the type of problem (classification or regression) the MATLAB command is,
\begin{verbatim}
>> f = PMKL(x,y)
\end{verbatim}
where if $\texttt{y}$ only contains two unique values, the algorithm defaults to classification. Otherwise, the algorithm defaults to regression.

\textbf{Manual Selection of Classification or Regression}\\
To run the PMKL algorithm with all default values, for the samples $\texttt{x}$ and outputs $\texttt{y}$ but manually select the type of problem, the MATLAB command is,
\begin{verbatim}
>> f = PMKL(x,y,Type)
\end{verbatim}
where \texttt{Type = 'Classification'} for classification or \texttt{'Regression'} for regression.

\textbf{Manually Specifying the Penalty C}\\
To run the PMKL algorithm with all default values except for Type and the penalty term $\texttt{C}$, for the samples $\texttt{x}$ and outputs $\texttt{y}$ the MATLAB command is,
\begin{verbatim}
>> f = PMKL(x,y,Type,C)
\end{verbatim}
where the user must select a $\texttt{C}>0$.  It is recommended that the value of $\texttt{C}$ be selected via k-fold cross-validation with data split into training and validation sets.

\textbf{Manually Specifying Additional Parameters}\\
For help generating the params structure we have included the paramsTK.m function which allows the user to generate a params structure for TK kernels as follows.
\begin{verbatim}
>> params = paramsTK(degree,delta,epsilon,maxit,tol)
\end{verbatim}
An empty matrix can be used for any input where the default value is desired, and using the paramsTK framework is recommended when modifying the default values.

\textbf{The Evaluate Subroutine}
Once the optimal kernel function has been learned, you can evaluate the predicted output of a set of samples using the following function.
\begin{verbatim}
>> yPred = evaluatePMKL(f,xTest)
\end{verbatim}
The output of \texttt{evaluatePMKL} are the predicted outputs of the optimal support vector machine, trained on the data $x$ and $y$ with the designated kernel function optimized by the PMKL subroutine.
% $>> \text{yPred = evaluatePMKL(f,xTest)}$.
\begin{itemize}
    \item The input \texttt{f} is an internal data structure output from the PMKL function.
    \item $\texttt{xTest} \in \R^{n \times m}$ is a matrix of $n$ rows corresponding to the number of features and $m$ columns corresponding to the number of samples where $\texttt{xTest(:,i)}$ is the i'th sample,
    \item The output $\texttt{yPred} \in \R^{1 \times m}$ is a vector with $m$ columns corresponding to the number of samples in $\texttt{xTest}$.
\end{itemize}

%
%\vskip 0.2in
%\bibliography{kernel_methods}
\bibliography{kernel_methods}

\end{document}